%% file: refined-approachability.tex
\documentclass[11pt]{amsart}
\usepackage[T1]{fontenc}
\usepackage{amssymb}
\usepackage{fullpage}
\usepackage[sort&compress]{natbib}
\usepackage{xcolor}
\definecolor{HalfGray}{gray}{0.55}
\definecolor{OliveGreen}{rgb}{0,.35,0}
\definecolor{webbrown}{rgb}{.6,0,0}
\definecolor{BrightViolet}{rgb}{0.5,0.2,0.8}
\definecolor{Maroon}{cmyk}{0, 0.87, 0.68, 0.32}
\definecolor{RoyalBlue}{cmyk}{1, 0.50, 0, 0.25}
\definecolor{Black}{cmyk}{0, 0, 0, 0}

\definecolor{ccccccc}{RGB}{204,204,204}
\definecolor{c808080}{RGB}{128,128,128}
\definecolor{c999999}{RGB}{153,153,153}
\definecolor{ce6e6e6}{RGB}{230,230,230}

\usepackage[unicode]{hyperref}
\hypersetup{
colorlinks=true,
linktocpage=true,
pdfstartview=FitH,
breaklinks=true,
pdfpagemode=UseNone,
pageanchor=true,
pdfpagemode=UseOutlines,
plainpages=false,
bookmarksnumbered,
bookmarksopen=false,
bookmarksopenlevel=1,
hypertexnames=true,
pdfhighlight=/O,
urlcolor=Maroon, linkcolor=RoyalBlue, citecolor=OliveGreen, 
pdftitle={},
pdfauthor={},
pdfsubject={},
pdfkeywords={},
pdfcreator={pdfLaTeX},}
\usepackage{pgfplots}
\usepackage{amsthm,amsmath,enumerate,bbm,tikz,tabularx,multicol,breqn}
\usetikzlibrary{matrix,arrows,calc,patterns}
\usepgfplotslibrary{groupplots,dateplot}
\usetikzlibrary{patterns,shapes.arrows}
\pgfplotsset{compat=newest}
\newtheorem{theorem}{Theorem}[section]
\newtheorem{proposition}[theorem]{Proposition}

\newtheorem{lemma}[theorem]{Lemma}
\theoremstyle{definition}
\newtheorem{definition}[theorem]{Definition}

\theoremstyle{remark}
\newtheorem{remark}[theorem]{Remark}
\DeclareMathOperator*{\argmax}{arg\,max}
\DeclareMathOperator*{\argmin}{arg\,min}
\DeclareMathOperator*{\proj}{proj}
\renewcommand{\bar}{\overline}

\author{Joon Kwon}
\address{INRAE \& AgroParisTech}
\email{joon.kwon@inrae.fr}
\date{\today}
\title[Refined approachability algorithms]{Refined approachability algorithms and application to regret minimization with global costs}
\begin{document}

\begin{abstract}
  \input{abstract}
\end{abstract}
\maketitle


\input{body}

\section*{Acknowledgements}
\input{acknowledgements}

\bibliographystyle{abbrvnat}
\bibliography{/home/joon/math/bibliography/references}
\newpage
\appendix
\addtocontents{toc}{\protect\setcounter{tocdepth}{0}}

\input{appendix}

\end{document}

%% file: abstract.tex
Blackwell's approachability is a framework where two players, the
Decision Maker and the Environment, play a repeated game with
vector-valued payoffs. The goal of the Decision Maker is to make the
average payoff converge to a given set called the target.  When this
is indeed possible, simple algorithms which guarantee the convergence
are known. This abstract tool was successfully used for the
construction of optimal strategies in various repeated games, but also
found several applications in online learning.  By extending an
approach proposed by \cite{abernethy2011blackwell}, we construct and
analyze a class of Follow the Regularized Leader algorithms (FTRL) for
Blackwell's approachability which are able to minimize not only the
Euclidean distance to the target set (as it is often the case in the
context of Blackwell's approachability) but a wide range of
distance-like quantities. This flexibility enables us to apply these
algorithms to closely minimize the quantity of interest in various
online learning problems. In particular, for regret minimization with
$\ell_p$ global costs, we obtain the first bounds with explicit
dependence in $p$ and the dimension $d$.


%% file: body.tex
\input{introduction}
\input{approachability}
\input{ftrl}
\input{global-costs}

%% file: introduction.tex
\section{Introduction}
\label{sec:introduction}

One of the foundational results of game theory is von Neumann's
minimax theorem which characterizes the highest payoff that each
player of a finite zero-sum game can guarantee regardless of the
opponent's strategy.  In the seminal works of
\cite{blackwell1956analog,blackwell1954controlled}, a surprising
extension of this result was proposed in the context of repeated games
with vector-valued payoffs. The so-called Blackwell's condition
characterizes the convex sets that the player can guarantee to
asymptotically reach, regardless of the opponent's actions. In the
case of non-convex sets, this condition remains sufficient.  When the
above condition is satisfied for a given set called the \emph{target},
the original algorithm proposed by Blackwell guarantees that the
average vector-valued payoff converges to (\emph{approaches}) the
target set at rate \(O(1/\sqrt{T})\), where $T$ is the number of
rounds of the repeated play.  This topic is now called Blackwell's
approachability.

This framework was used for the construction of optimal strategies in
repeated games as in \cite{kohlberg1975optimal}, see also the survey
work by \cite{perchet2014approachability} and references therein.  Beyond
the field of game theory, this tool has been noticed by the machine
learning community and used for constructing and analyzing algorithms
for various online decision problems such as regret
minimization \citep{cesa2006prediction}, asymptotic calibration \citep{dawid1982well,foster1998asymptotic}, regret minimization with
variable stage duration \citep{mannor2008regret} or with global cost
functions \citep{even2009online}. However, one drawback of
using Blackwell's approachability is that algorithms then usually
minimize the \emph{Euclidean distance} of the average payoffs to the target
set, which is seldom the exact quantity of interest in online learning
applications. One of the main objectives of the present work is to
provide a flexible class of algorithms which are able to minimize various distance-like
quantities, and not only the Euclidean distance.

Several alternative approachability algorithms were also proposed, including 
potential-based algorithm \citep{hart2001general} which generalize the
Euclidean projection involved in Blackwell's algorithm, and response-based
algorithms \citep{bernstein2015response} which avoid the projection
altogether. Besides, an important scheme
used in several works is the conversion of regret minimization
algorithms into approachability algorithms
\citep{abernethy2011blackwell, shimkin2016online, mannor2014approachability}.

Regret minimization was introduced by \cite{hannan1957approximation}
and is a sequential decision problem where the Decision Maker aims at
minimizing the difference between its payoff and the highest payoff in
hindsight given by a constant strategy.  The link between
approachability and regret minimization was already noticed by
\cite{blackwell1954controlled} who reduced regret minimization to an
approachability problem.  \cite{hart2001general} proposed an
alternative reduction and constructed a whole family of regret
minimization algorithm using potential-based approachability
algorithms. \cite{gordon2007no} extended the potential-based approach
to a wider range of regret minimization problems, seen as
approachability problems. Conversely, regret minimizing algorithms
have been converted into approachability algorithms
\citep{gordon2007no,abernethy2011blackwell,perchet2015exponential,shimkin2016online}.

It is worth noting that modern variants of the Regret Matching
algorithm, which is a special case of potential-based approachability
algorithms \citep{hart2000simple,hart2001general}, are today the state-of-the-art online learning
algorithms for Nash equilibrium computation in large zero-sum games \citep{zinkevich2007regret,tammelin2015solving}.
\subsection{Related work}
\label{sec:related-work}
In \cite{perchet2015exponential}, the Exponential Weights Algorithm, which
is a central regret minimization algorithm, is adapted to
approachability, and the resulting algorithm minimizes the
\(\ell_\infty\) distance to the target set.

The conversion scheme presented in \cite{abernethy2011blackwell} deals
with online linear optimization algorithms which are transposed into
the approachability of convex cone target sets, and the
associated guarantee is an upper bound on the Euclidean distance to
the target set. An extension to all convex target sets is also given,
which involves the adding of a dimension. 

A closely related work is \cite{shimkin2016online} where a conversion
from online \emph{convex} optimization algorithm to approachability of
\emph{bounded} convex sets is presented, which guarantees an upper bound on the
distance to the target set measured with the Euclidean norm or possibly any
other norm.

One of the applications of approachability is the problem of regret
minimization with global costs, introduced in \cite{even2009online} and
already analyzed as an approachability problem. This problem was
further studied in \cite{rakhlin2011online,bernstein2015response}, and
in a recent paper \citep{liu2020improved}, the authors used the conversion scheme
from \cite{shimkin2016online} to construct and analyze algorithms for
this problem.

A recent paper \citep{farina2020faster} proposes an extension of
\cite{abernethy2011blackwell} by introducing predictive approchability
algorithms. The proposed construction shares similarities with the
present work but focuses on variants of Regret Matching, and only
derives upper bounds on Euclidean distances to target sets---see
\cite[Proposition 2]{farina2020faster}, whereas we consider a more
general class of quantities in Definition~\ref{def:support-function} and
Theorem~\ref{thm:approachability-mirror-descent} below.

\subsection{Contributions}
\label{sec:contributions}
\begin{itemize}
\item We consider a class of Follow the Regularized Leader algorithms
  (FTRL) which we convert from regret minimization to approachability.
  The conversion scheme we use is a refinement of
  \cite{abernethy2011blackwell}, which itself is an extension of
  \cite{gordon2007no}, and the algorithms that we obtain are capable
  of minimizing not only the Euclidean distance to the target set as
  in \cite{abernethy2011blackwell}, but the distance measured by an
  arbitrary norm, or even more general distance-like quantities. This
  flexibility will prove itself useful in the construction of tailored
  algorithms with tight bounds for various problems.
\item For the problem of regret minimization with global cost, we
construct algorithms for arbitrary norm cost functions and obtain
novel guarantees. In particular, for \(\ell_p\) norm cost functions
(\(p>1\)), we obtain the first explicit
regret bounds that depend on \(p\) and the dimension $d$, and which
recovers, in the special case $p=\infty$ the best known
$O(\sqrt{(\log d)/T})$ bound.
\end{itemize}
\subsection{Summary}
In Section~\ref{sec:appr-conv-cones}, we present a model of
approachability with target sets which are closed convex cones.  In
Section~\ref{sec:class-foll-regul}, we define a class of FTRL
algorithms and derive a general guarantee.  In
Section~\ref{sec:regr-minim-with}, we recall the problem of regret
minimization with global cost functions and relate it to our
approachability framework and FTRL algorithms. In the special case of
\(\ell_p\) norm cost functions, we derive regret bounds with explicit
dependence in $d$ and $p$.  In Appendix~\ref{sec:repr-blackw-algor},
we recall Blackwell's algorithm and prove that it belongs to the class
of algorithms defined in Section~\ref{sec:class-foll-regul}.  In
Appendix~\ref{sec:model-with-mixed}, we present a variant of the model
from Section~\ref{sec:appr-conv-cones}, where the Decision Maker may
choose its actions at random from a finite set. We then define
corresponding FTRL algorithms and provide guarantees in expectation,
with high probability and almost-surely.  In
Appendices~\ref{sec:online-comb-optim}
and~\ref{sec:internal-swap-regret}, we recall the problems of online
combinatorial optimization and internal/swap regret respectively,
their reductions to approachability problems, and demonstrate that a
carefully chosen FTRL algorithm recover the known optimal bounds.
\subsection{Notation}
$\mathbb{R}_+^*$ denotes the set of positive real numbers.
\(d\geqslant 2\) will always denote an integer.  All vector spaces will
be of finite dimension.  For \(p\in [1,+\infty]\), we denote
\(\left\|\,\cdot\,\right\|_p\) the \(\ell_p\) norm, meaning for
\(x\in \mathbb{R}^d\), \(\left\| x \right\|_p=\left( \sum_{i=1}^d\left|
x_i \right|^p \right)^{1/p}\) for \(p<+\infty\) and \(\left\| x
\right\|_{\infty}=\max_{1\leqslant i\leqslant d}\left| x_i \right|\).  For a given norm
\(\left\|\,\cdot\,\right\|_{}\) in a vector space, the dual norm
\(\left\|\,\cdot\,\right\|_{*}\) is defined by
\(\left\| y \right\|_{*}=\sup_{\left\| x \right\|_{}\leqslant 1}\left|
\left< y , x \right> \right|\). 
Denote \(\Delta_d\) the unit simplex of \(\mathbb{R}^d\): \(\Delta_d=\left\{ x\in \mathbb{R}_+^d,\ \sum_{i=1}^dx_i=1 \right\}\).
For a sequence
\((r_t)_{t\geqslant 1}\) of vectors, we denote
\(\bar{r}_T=\frac{1}{T}\sum_{t=1}^Tr_t\) the average of the \(T\) first
terms (\(T\geqslant 1\)).  If \(\mathcal{X}\) a subset of a vector space,
\(I_{\mathcal{X}}\) denotes the convex indicator of \(\mathcal{X}\), in other
words: \(I_{\mathcal{X}}(x)=0\) if \(x\in \mathcal{X}\) and
\(I_{\mathcal{X}}(x)=+\infty\) otherwise.  If a vector \(x_t\in
\mathbb{R}^d\) is denoted with an index (\(t\) in this example), its
components are denoted with an additional index as follows:
\(x_t=(x_{ti})_{1\leqslant i\leqslant d}\).

%% file: approachability.tex
\section{Approachability of convex cones}
\label{sec:appr-conv-cones}
We introduce a simple repeated game with vector-valued payoffs between
two players (the Decision Maker and the Environment) with a closed
convex cone target set for the Decision Maker. We then state a few
properties about closed convex cones and support functions.
\subsection{Model}
\label{sec:model}
Let \(\mathcal{V}\) be a finite-dimensional vector space and denote \(\mathcal{V}^*\)
its dual. The latter will be the \emph{payoff space}.
Let \(\mathcal{A}\) and \(\mathcal{B}\) be the \emph{action sets} for the
Decision Maker and the Environment respectively, about
which we assume no special structure.
Let \(r\colon \mathcal{A}\times \mathcal{B}\to \mathcal{V}^*\) be a
vector-valued \emph{payoff function}.
The game is played as follows. At time \(t\geqslant 1\), 
\begin{itemize}
\item the Decision Maker chooses action \(a_t\in \mathcal{A}\); \item
the Environment chooses action \(b_t\in \mathcal{B}\); \item the
Decision Maker observes \emph{vector payoff} \(r_t:=r(a_t,b_t)\in
\mathcal{V}^*\).  \end{itemize} 

We allow the Environment to be \emph{adversarial}\footnote{In other words,
action \(b_t\) chosen by the
Environment may depend on anything that has happened before it is
chosen, including \(a_t\).}. 

The problem involves a \emph{target set} \(\mathcal{C}\subset \mathcal{V}^*\)
which we assume to be a closed convex cone\footnote{For the case where target set is a closed convex set but
not a cone, we refer to~\cite[Section 4 \& Lemma 14]{abernethy2011blackwell} where a
conversion scheme into an auxiliary problem where the target is a cone is presented.}.
 The goal is to construct algorithms
which guarantee that the average payoff
\(\bar{r}_T:=\frac{1}{T}\sum_{t=1}^Tr_t\) is \emph{close} to the target
\(\mathcal{C}\) in a sense that will be made precise.

The above model does not allow the Decision Maker to choose
actions at random. Such a model is presented in
Appendix~\ref{sec:model-with-mixed}.
\subsection{Generator of a closed convex cone}
\label{sec:generator}
We now introduce a key notion of this work which will be used in
Section~\ref{sec:support-functions} to define the class of quantities that will be
minimized by the algorithms defined in Section~\ref{sec:defin-analys-algor}.
Definitions and properties about closed convex cones are
gathered in Appendix~\ref{sec:prop-clos-conv}.

\begin{definition}
\label{def:generators}
Let $\mathcal{C}$ be a closed convex cone. A set $\mathcal{X}$ is a \emph{generator} of $\mathcal{C}$
if it is convex, compact and if $\mathbb{R}_+\mathcal{X}=\mathcal{C}$.
\end{definition}


The following proposition gives three examples of generators. The
second example demonstrates that a generator always exists. The proof
is given in Appendix~\ref{sec:proof-generator-examples}.
\begin{proposition}
\label{prop:generator-examples}
Let \(\mathcal{W}\) be the ambient finite-dimensional vector space.
\begin{enumerate}[(i)]
\item\label{item:orthant} If $\mathcal{W}=\mathcal{W}^*=\mathbb{R}^d$, the negative orthant $\mathbb{R}_-^d$ is a closed convex cone and $(\mathbb{R}_-^d)^\circ =\mathbb{R}_+^d$.
Moreover, $\Delta_d$ is a generator of $\mathbb{R}_+^d$.
\item\label{item:generator-always} Let $\mathcal{C}\subset \mathcal{W}$ be a closed convex cone, $\left\|\,\cdot\,\right\|$ a norm on $\mathcal{W}$, and $\mathcal{B}$ the closed unit ball with respect to $\left\|\,\cdot\,\right\|$. Then, $\mathcal{B}\cap \mathcal{C}$ is a generator of $\mathcal{C}$.
\item\label{item:Z-generator-of-what} If $\mathcal{X}$ is a nonempty convex compact subset of $\mathcal{W}$, then $\mathcal{X}$ is a generator of $\mathcal{X}^{\circ \circ}=\mathbb{R}_+\mathcal{X}$.
\end{enumerate}
\end{proposition}
\subsection{Support functions}
\label{sec:support-functions}
We now present support functions which will be used in Section~\ref{sec:defin-analys-algor}
to express the quantities that will be minimized by our algorithms.
\begin{definition}
\label{def:support-function}
For a nonempty subset $\mathcal{X}\subset \mathcal{V}$, 
the application $I_{\mathcal{X}}^*:\mathcal{V}^*\to \mathbb{R}\cup\{+\infty\}$ defined by
\[ I_{\mathcal{X}}^*(y)=\sup_{x\in \mathcal{X}}\left< y , x \right> ,\quad y\in \mathcal{V}^*, \]
is called the \emph{support function} of $\mathcal{X}$.
\end{definition}
The support function can be written as the Legendre--Fenchel transform
of the indicator function of the set \(\mathcal{X}\). It is therefore convex. Moreover,
in the case where \(\mathcal{X}\) is a generator of the polar cone
\(\mathcal{C}^\circ\) of some closed convex cone \(\mathcal{C}\subset \mathcal{V}^*\), the properties of
\(I_{\mathcal{X}}^*\) make it suitable for measuring how far a point of \(\mathcal{V}^*\) is
from \(\mathcal{C}\). Indeed, it is easy to check that
\(I_{\mathcal{X}}^*\) is then real-valued, continuous, and that for all points \(y\in \mathcal{V}^*\),
\[ I_{\mathcal{X}}^*(y)\leqslant 0\quad \Longleftrightarrow \quad y\in \mathcal{C}. \]
The following proposition demonstrates that the distance to a closed convex cone
\(\mathcal{C}\) with respect to an arbitrary norm can be written as a
support function. It is an is an extension of Lemma 13
in~\cite{abernethy2011blackwell} to an arbitrary norm. The proof is
given in Appendix~\ref{sec:proof-support}.
\begin{proposition}
\label{prop:distance-support}
Let $\mathcal{C}$ be a closed convex cone in $\mathcal{V}^*$, $\left\|\,\cdot\,\right\|_{}$ a norm on $\mathcal{V}$
and $\left\|\,\cdot\,\right\|_{*}$ its dual norm on $\mathcal{V}^*$. Then,
\[ \inf_{y'\in \mathcal{C}}\left\| y'-y \right\|_{*}=I_{\mathcal{B}\cap \mathcal{C}^\circ }^*(y),\quad y\in \mathcal{V}^*, \]
where $\mathcal{B}$ is the closed unit ball for $\left\|\,\cdot\,\right\|_{}$.
\end{proposition}
\subsection{Blackwell's condition}
\label{sec:blackwells-condition}
In the case of convex sets, Blackwell's condition~\citep{blackwell1956analog} is a characterization of the target
sets to which the Decision Maker can guarantee a convergence. We here
present the special case of convex cones, which will be used in the
construction and the analysis of the algorithms in Section~\ref{sec:defin-analys-algor}.
\begin{definition}[Blackwell's condition for convex cones]
\label{def:B-set}
A closed convex cone $\mathcal{C}$ of the payoff space $\mathcal{V}^*$ is a \emph{B-set} for the game $(\mathcal{A},\mathcal{B},r)$ if
\[ \forall x\in \mathcal{C}^\circ ,\ \exists\,  a(x)\in \mathcal{A},\ \forall b\in \mathcal{B},\quad \left< r(a(x),b) , x \right> \leqslant 0. \]
Such an application $a:\mathcal{C}^\circ \to \mathcal{A}$ is called a \emph{$(\mathcal{A},\mathcal{B},r,\mathcal{C})$-oracle}.
\end{definition}
The geometric interpretation of this condition is that for any given
hyperplane containing the target, the Decision Maker has an action
which forces the payoff vector to belong the same side of the
hyperplane as the target, regardless of the Environment's action.

In some situations, it is easier to establish the following equivalent
dual condition. The proof is given in Appendix~\ref{sec:proof-dual-condition} for completeness.
\begin{proposition}[Blackwell's dual condition]
  \label{prop:dual-condition}
  We assume that $\mathcal{A}$, $\mathcal{B}$ are convex sets of
  finite dimensional vectors spaces, such that $\mathcal{A}$ is compact,
  and that the payoff function
  $r\colon \mathcal{A}\times \mathcal{B}\to \mathcal{V}^*$ is
  bi-affine.  Then, a closed convex cone $\mathcal{C}$ of the payoff space
  $\mathcal{V}^*$ is a B-set for the game
  $(\mathcal{A},\mathcal{B},r)$ if, and only if
\[ \forall b\in \mathcal{B},\ \exists a\in \mathcal{A},\quad r(a,b)\in \mathcal{C}.  \]
\end{proposition}


%% file: ftrl.tex
\section{A class of FTRL algorithms}
\label{sec:class-foll-regul}
We define a class of Follow the Regularized Leader algorithms (FTRL)
which are transposed from regret minimization, and which 
 guarantee, when the target is a B-set, that the
average payoff converges to the target set, the convergence being
measured in a sense that will be made precise.
\subsection{Regularizers}
\label{sec:regularizers}
We first introduce regularizers functions and the notion of strong convexity needed for the
definition and the analysis of FTRL algorithms~\citep{shalev2007online,shalev2011online,bubeck2011introduction}, which
are also known as \emph{dual averaging}~\citep{nesterov2009primal} in the context of optimization. These
are classic: basic properties, proofs and important examples are recalled in Appendix~\ref{sec:proofs-prop-regul}. Again,
\(\mathcal{V}\) and \(\mathcal{V}^*\) are finite-dimensional vectors spaces and
\(\mathcal{X}\) is a nonempty convex compact subset of \(\mathcal{V}\).
We recall that the \emph{domain} \(\operatorname{dom}h\) of a function
\(h:\mathcal{V}\to \mathbb{R}\cup\{+\infty\}\) is the set of points where it has finite values.
\begin{definition}
A convex function $h:\mathcal{V}\to \mathbb{R}\cup\{+\infty\}$ is a \emph{regularizer} on $\mathcal{X}$
if it is strictly convex, lower semicontinuous, and has $\mathcal{X}$ as domain.
\end{definition}

\begin{definition}
\label{def:strong-convexity}
Let $h:\mathcal{V}\to \mathbb{R}\cup\{+\infty\}$ be a function, $\left\|\,\cdot\,\right\|$ a norm on $\mathcal{V}$, and $K>0$. $h$ is $K$-strongly convex
with respect to $\left\|\,\cdot\,\right\|_{}$ if for all $x,x'\in \mathcal{V}$ and $\lambda\in [0,1]$,
\begin{equation}
\label{eq:definition-strong-convexity}
h(\lambda x+(1-\lambda)x')\leqslant \lambda h(x)+(1-\lambda)h(x')-\frac{K\lambda(1-\lambda)}{2}\left\| x'-x \right\|_{}^2.
\end{equation}
\end{definition}
\subsection{Definition and analysis of the algorithm}
\label{sec:defin-analys-algor}
We now construct the FTRL algorithms for the model
introduced in Section~\ref{sec:model} and establish guarantees.

Let \(\mathcal{C}\) be a B-set for the game \((\mathcal{A},\mathcal{B},r)\) and \(a:\mathcal{C}^\circ \to \mathcal{A}\) a \((\mathcal{A},\mathcal{B},r,\mathcal{C})\)-oracle.
Let \(\mathcal{X}\subset \mathcal{V}\) be a generator of \(\mathcal{C}^\circ\),
\(h\) a regularizer on \(\mathcal{X}\), and \((\eta_t)_{t\geqslant 1}\) a positive
sequence of parameters.
The associated algorithm is then defined for \(t\geqslant 1\) as:
\begin{align*}
\text{compute}\quad x_t&=\argmax_{x\in \mathcal{X}}\left\{ \left< \eta_{t-1}\sum_{s=1}^{t-1}r_s ,x \right> -h(x)  \right\}\\
\text{compute}\quad a_t&=a\left( x_t \right)\\
\text{observe}\quad r_t&=r(a_t,b_t),
\end{align*}
where the first line is well-defined thanks to the basic properties of regularizers
gathered in Proposition~\ref{prop:lengendre}. We prove in
Appendix~\ref{sec:repr-blackw-algor} that Blackwell's original
algorithm belongs to this class.

The above definition of \(x_t\) can be interpreted as the action
played by a FTRL algorithm in an online linear
optimization problem with action set \(\mathcal{X}\) and payoff vectors
\((r_t)_{t\geqslant 1}\). We state in the following lemma the classical
\emph{regret bound} guaranteed by such an algorithm~\citep{shalev2007online,shalev2011online,bubeck2011introduction}. The proof is
given in Appendix~\ref{sec:proof-regret-bound} for completeness. This regret bound will then be
\emph{converted} in Theorem~\ref{thm:approachability-mirror-descent} into an upper bound on
\(I_{\mathcal{X}}^*(\bar{r}_T)\), thus providing a guarantee for the
approachability game. This conversion is an extension of the scheme
introduced in~\cite{abernethy2011blackwell},
which gives
approachability algorithms which minimize the Euclidean distance of
the average payoff to the target set. Our approach is more general
as it allows, by the choice of the generator \(\mathcal{X}\),
to minimize a whole class of distance-like quantities.

The conversion is here applied to FTRL algorithms,
but could have been applied to any online linear optimization
algorithm.

In a recent paper \citep{farina2020faster}, the authors also propose a similar
extension of \cite{abernethy2011blackwell} which is however less general, as
they only consider generators which contain
$\mathcal{C}^\circ \cap \mathcal{B}_2$, where $\mathcal{B}_2$ is the
Euclidean ball.

\begin{lemma}[Regret bound]
\label{lm:regret-bound}
Let $\Delta,K,M>0$, $\left\|\,\cdot\,\right\|$ a norm on
$\mathcal{V}$, and $\left\|\,\cdot\,\right\|_{*}$ its dual norm on $\mathcal{V}^*$. We assume:
\begin{enumerate}[(i)]
\item $\max_{x\in \mathcal{X}}h(x)-\min_{x\in \mathcal{X}}h(x)\leqslant \Delta$,
\item $h$ is $K$-strongly convex with respect to  $\left\|\,\cdot\,\right\|$,
\item $\left\| r_t \right\|_* \leqslant M$ for all $t\geqslant 1$.
\end{enumerate}
Then, the choice $\eta_t=\sqrt{\Delta K/M^2t}$ (for $t\geqslant 1$) guarantees
\[ \forall T\geqslant 1,\quad \max_{x\in
    \mathcal{X}}\sum_{t=1}^T\left< r_t , x \right>
  -\sum_{t=1}^T\left< r_t , x_t \right> \leqslant
  2M\sqrt{\frac{\Delta T}{K}}. \]
\end{lemma}

The following theorem provides upper bounds on
\(I_{\mathcal{X}}^*(\bar{r}_T)\) (where \(\bar{r}_T=\frac{1}{T}\sum_{t=1}^Tr_t\) is
the average payoff) and not only the Euclidean distance from \(\bar{r}_T\)
to \(\mathcal{C}\), which is a special case---see
Proposition~\ref{prop:distance-support}. Therefore, the choice of the generator
\(\mathcal{X}\) determines the quantity that is minimized by the
algorithm. We present in
Sections~\ref{sec:regr-minim-with} and Appendices \ref{sec:online-comb-optim} and \ref{sec:internal-swap-regret} examples of problems where a judicious choice of generator
\(\mathcal{X}\) allows \(I^*_{\mathcal{X}}(\bar{r}_T)\) to be 
equal (or close) to the quantity the Decision Maker actually aims at minimizing and
therefore yields tailored algorithms.
\begin{theorem}
\label{thm:approachability-mirror-descent}
Let $\Delta,K,M>0$, $\left\|\,\cdot\,\right\|$ a norm on
$\mathcal{V}$, and $\left\|\,\cdot\,\right\|_{*}$ its dual norm on $\mathcal{V}^*$. We assume:
\begin{enumerate}[(i)]
\item $\max_{x\in \mathcal{X}}h(x)-\min_{x\in \mathcal{X}}h(x)\leqslant \Delta$,
\item $h$ is $K$-strongly convex with respect to $\left\|\,\cdot\,\right\|$,
\item $\left\| r(a,b) \right\|_* \leqslant M$ for all $a\in \mathcal{A}$ and $b\in \mathcal{B}$.
\end{enumerate}
Then the above algorithm guarantees, with the choice $\eta_t=\sqrt{\Delta
  K/M^2t}$ (for $t\geqslant 1$), against any sequence
of actions $(b_t)_{t\geqslant 1}$ chosen by the Environment,
\[ \forall T\geqslant 1,\quad  I_{\mathcal{X}}^*\left( \bar{r}_T \right)\leqslant 2M\sqrt{\frac{\Delta}{KT}}.  \]
\end{theorem}
\begin{proof}
The regret from Lemma~\ref{lm:regret-bound} is the following quantity: 
\[ \operatorname{Reg}_T=\max_{x\in \mathcal{X}}\sum_{t=1}^T\left< r_t , x \right> -\sum_{t=1}^T\left< r_t , x_t \right>. \]
The first term above can be written
\[ \max_{x\in \mathcal{X}}\sum_{t=1}^T\left< r_t , x \right> =T\cdot \max_{x\in \mathcal{X}}\left< \frac{1}{T}\sum_{t=1}^T r_t, x \right> =T\cdot I_{\mathcal{X}}^*\left( \bar{r}_T \right),  \]
whereas the second sum is nonpositive because each term is. Indeed, by definition of the algorithm, and because $a$ is a $(\mathcal{A},\mathcal{B},r,\mathcal{C})$-oracle,
\[ \left< r_t , x_t \right> =\left< r(a_t,b_t) , x_t \right> =\left< r(a(x_t),b_t) , x_t \right> \leqslant 0. \]
Therefore $I_{\mathcal{X}}^*(\bar{r}_T)\leqslant \frac{1}{T}\operatorname{Reg}_T$ and the regret bound from Lemma~\ref{lm:regret-bound} gives the result.
\end{proof}

In Appendix~\ref{sec:model-with-mixed}, we present a variant of the present model
where the Decision Maker can choose its actions at random. The above
guarantee is transposed into guarantees in expectation, in
high-probability (using the Azuma--H\oe ffding inequality), and into
almost-sure convergence (using a Borel--Cantelli argument).

%% file: global-costs.tex
\section{Regret minimization with global costs}
\label{sec:regr-minim-with}
The problem of regret minimization with global costs was introduced in~\cite{even2009online}. It is an adversarial online learning problem
motivated by load balancing and job scheduling, where at each step,
the Decision Maker first chooses a distribution (task allocation) over
\(d\) machines, and then observes the cost of using each machine, which
may be different for each machine and each step. The goal of the
Decision Maker is to minimize, not the sum of the cumulative costs of
using each machine, but a given function of the vector of cumulative
costs. A typical example of such \emph{global cost} function is the
\(\ell_p\) norm, which includes as special cases the sum of the costs
(for \(p=1\)), as well as the makespan i.e.\ the highest cumulative cost
(for \(p=\infty\)).  A very common approach for this type of problem is
to focus on competitive ratio~\citep{borodin1998online,azar1993online,molinaro2017online}. We instead
follow~\cite{even2009online} and aim at minimizing the regret.

In the seminal paper by~\cite{even2009online}, the authors introduce a
reduction of the problem to an approachability game and obtain a
regret bound of order \(O((\log d)/\sqrt{T})\) for the
\(\ell_{\infty}\) cost function. For general convex cost functions,
the authors present a regret bound that reads $\sqrt{d/T}$; however,
this expression does not reflect the true dependency of the bound in
the number \(d\) of machines, as this bound also involves several 
Lipschitz constants that depend on the cost function, and which may also depend on
\(d\), as it is the case for \(\ell_p\) cost functions. In a
theoretical work, \cite{rakhlin2011online} proved that the regret
bound can be improved to \(O(\sqrt{(\log d)/T})\) in the \(\ell_{\infty}\)
case, but no algorithm achieving this bound was
provided. \cite{bernstein2015response} also studied alternative
algorithms for minimizing regret with global cost but no explicit
bound was given. In a recent paper by \cite{liu2020improved}, new
algorithms are proposed, based on a technique for adapting online
convex optimization algorithms to approachability
games~\citep{shimkin2016online}, and regret bounds for monotone norms
cost functions (which include \(\ell_p\) norms) are derived. The
bounds are abstract, except for the \(\ell_{\infty}\) case where the
algorithm achieve the best known \(O(\sqrt{(\log d)/T})\) bound in
addition of being the first such algorithm to run in polynomial
time. Besides, more general problems than the one we consider below
are studied in~\cite{azar2014sequential,mannor2014approachability} and
both provide algorithms with convergence rate $T^{-1/4}$.

In this section, we apply the tools introduced in
Sections~\ref{sec:appr-conv-cones} and \ref{sec:class-foll-regul} to
construct and analyze new algorithms for this problem. Although our
approach applies to general norm cost functions (unlike
\cite{liu2020improved} which assumes the norm to be monotone), we
focus in Section~\ref{sec:an-algorithm-ell_p} on $\ell_p$ norms ($p>1$) to
obtain explicit regret bounds in
Theorem~\ref{thm:global_costs_lp_norms}, which, in the special case
\(p=\infty\), recovers the best known \(O(\sqrt{(\log d)/T})\)
bound. To the best of our knowledge, these are the first regret bounds
for \(\ell_p\) norm cost functions with explicit dependence in $d$ and $p$.

We use the reduction of the problem to an approachability game
from~\cite{even2009online}. We then choose a generator of the polar of
the target set based on a specially crafted norm on the payoff space,
which then enables us to bound the regret with cost functions by a
support function. Then, in the case of \(\ell_p\) cost functions, the
explicit regret bounds are derived with the help of a carefully chosen
regularizer.

\subsection{Problem statement}
\label{sec:problem-statement}
Let \(d\geqslant 2\) be an integer and \(\left\|\,\cdot\,\right\|\) a norm on
\(\mathbb{R}^d\). Recall that $\Delta_d$ denotes the unit simplex of
$\mathbb{R}^d$ and is identified with the set of probability
distributions over $\mathcal{I}$. For \(t\geqslant 1\),
\begin{itemize}
\item the Decision Maker chooses distribution \(a_t\in \Delta_d\);
\item the Environment chooses loss vector \(\ell_t\in [0,1]^d\).
\end{itemize}
The Decision Maker aims at minimizing the following average regret:
\[ \bar{\operatorname{Reg}}_T=\left\| \frac{1}{T}\sum_{t=1}^Ta_t\odot \ell_t\right\|-\min_{a\in \Delta_d}\left\| \frac{1}{T}\sum_{t=1}^Ta\odot \ell_t \right\|, \]
where \(\odot\) denotes the component-wise multiplication. At each
stage \(t\geqslant1\), the \(i\)-th component of vector
\(a_t\circ \ell\) is equal to \(a_{ti}\ell_i\) and corresponds to the
cost of using machine \(i\) for a fraction \(a_{ti}\) of the job.  The
regret is the difference between the actual global cost incurred by
the Decision Maker and the best possible global cost in hindsight for
a static distribution \(a\in \Delta_d\).  Important special cases
include the makespan which corresponds to
$\left\|\,\cdot\,\right\| =\left\|\,\cdot\,\right\|_{\infty}$: the
global cost is then the highest average cost over the machines; and
for $\left\|\,\cdot\,\right\| =\left\|\,\cdot\,\right\|_1$ the global
cost simply corresponds to the sum of the costs of all the machines,
and the problem then reduces to basic regret minimization.

\subsection{Reduction to an approachability game}
\label{sec:reduct-an-appr}
We recall the reduction given in~\cite[Section~4]{even2009online} of
the above problem to an approachability game which fits the model from Section~\ref{sec:appr-conv-cones}.

Consider the following action sets for the Decision Maker and the
Environment respectively: \(\mathcal{A}=\Delta_d\) and
\(\mathcal{B}=[0,1]^d\). Define the payoff function \(r:\Delta_d\times
[0,1]^d\to (\mathbb{R}^{d})^2\) as
\[ r(a,\ell)=(a\odot \ell,\ell),\quad a\in \Delta_d,\ \ell\in [0,1]^d, \]
and consider the following target set:
\[ \mathcal{C}=\left\{ (y,y')\in (\mathbb{R}^{d}_+)^2,\  \left\| y \right\| \leqslant \min_{a\in \Delta_d}\left\| a\odot y' \right\|  \right\}. \]
The payoff space is therefore \(\mathcal{V}^*=(\mathbb{R}^{d})^2\).

\begin{proposition}[\cite{even2009online}]
$\mathcal{C}$ is a closed convex cone. Moreover, it is a
B-set for the game $(\Delta_d,[0,1]^d,r)$.
\end{proposition}
\begin{proof}
We give the proof for the sake of completeness and essentially follow~\cite[Lemma~5 \& Theorem~6]{even2009online}.  $\mathcal{C}$ can be
written as
\[ \mathcal{C}=\left\{ (y,y')\in (\mathbb{R}^d_+)^2\mid \left\|
y \right\| -\min_{a\in \Delta_d}\left\| a\odot y' \right\|\leqslant 0
\right\}, \]
which then appears as a closed level set of a convex
function because $y\mapsto \left\| y \right\|$ is continuous and
convex for all norms, and because $y'\mapsto \min_{a\in
\Delta_d}\left\| a\odot y' \right\|$ is concave on $\mathbb{R}_+^d$ according
to~\cite[Lemma~22]{rakhlin2011online} and continuous as the minimum of
a family of continuous
functions. $\mathcal{C}$ is thus closed and convex, and because it is
is clearly closed by multiplication by a nonnegative scalar, it is a
closed convex cone.

We can now establish that $\mathcal{C}$ is a B-set for
the game $(\Delta_d,[0,1]^d,r)$ using Blackwell's dual condition
from Proposition~\ref{prop:dual-condition}, because the payoff
function $r$ is indeed bi-affine. Let $\ell\in [0,1]^d$ and consider $a_0=\argmin_{a\in \Delta_d}\left\| a\odot \ell \right\|$. Then, we clearly have $r(a_0,\ell)\in \mathcal{C}$, which concludes the proof.
\end{proof}

\begin{remark}[Computation of the oracle]
\label{rk:oracle}
As noted in \cite[Section~4]{even2009online} and \cite[Section~4.1]{liu2020improved}, a $(\Delta_d,[0,1]^d,r,\mathcal{C})$-oracle is given by
\[ a(z,z')=\argmin_{a\in \Delta_d}\sum_{i=1}^{d}\max_{}(0, z_ia_i+z'_i),\quad (z,z')\in \mathcal{C}^{\circ}, \]
which is a linear program with $O(d)$ variables and $O(d)$ constraints, which can thus be computed in polynomial time.
\end{remark}

\subsection{A special norm on the payoff space}
\label{sec:special-norm-payoff}
We now define a special norm on the payoff space \(\mathcal{V}^*\)
which will allow us to bound the regret from above with the help of a
support function, and will therefore provide the generator
of \(\mathcal{C}^\circ\) for defining the regularizer and constructing our
algorithm.

We introduce the following norm \(\left\|\,\cdot\,\right\|_{\mathcal{V}^*}\) whose
definition is based on the norm \(\left\|\,\cdot\,\right\|\) given in Section~\ref{sec:problem-statement}:
\[ \left\| (y,y') \right\|_{\mathcal{V}^*}=\left\| y \right\| +\max_{a\in \Delta_d}\left\| a\odot  y' \right\| ,\quad (y,y')\in \mathcal{V}^*=(\mathbb{R}^d)^2.  \]
It is easy to check that $\left\|\,\cdot\,\right\|_{\mathcal{V}^*}$ is
indeed a norm and we consider the associated the dual norm, defined on
\(\mathcal{V}\), which we denote \(\left\|\,\cdot\,\right\|_{\mathcal{V}}\).
We can now consider the following generator of \(\mathcal{C}^\circ\):
\(\mathcal{X}=\mathcal{B}\cap \mathcal{C}^\circ\), where \(\mathcal{B}\)
denotes the closed unit ball with respect to
\(\left\|\,\cdot\,\right\|_{\mathcal{V}}\). The following proposition
shows that this choice of \(\mathcal{X}\) makes the average regret \(\bar{\operatorname{Reg}}_T\) bounded
from above by \(I_{\mathcal{X}}^*(\bar{r}_T)\).
\begin{proposition}
  \label{prop:global-cost-support-function}
Let $(a_t)_{t\geqslant 1}$ and $(\ell_t)_{t\geqslant 1}$ be sequences of actions chosen
by the Decision Maker and the Environment respectively.
Denote for all $t\geqslant 1$, $r_t=r(a_t,\ell_t)$ the corresponding payoffs.
Then for all $T\geqslant 1$, the regret is bounded as
\[ \bar{\operatorname{Reg}}_T=\left\| \frac{1}{T}\sum_{t=1}^Ta_t\odot \ell_t \right\| -\min_{a\in \Delta_d}\left\| \frac{1}{T}\sum_{t=1}^Ta\odot \ell_t \right\| \leqslant I_{\mathcal{B}\cap \mathcal{C}^{\circ}}^*\left( \bar{r}_T \right), \]
where $\mathcal{B}$ denotes the closed unit ball associated with $\left\|\,\cdot\,\right\|_{\mathcal{V}}$.
\end{proposition}
\begin{proof}
Let $T\geqslant 1$ and denote $y=\frac{1}{T}\sum_{t=1}^Ta_t\odot \ell_t$ and $y'=\frac{1}{T}\sum_{t=1}^T\ell_t$. Let $(\tilde{y},\tilde{y}')\in \mathcal{C}$ be any vector from the target set.
Then, we can write
\begin{align*}
\bar{\operatorname{Reg}}_T&=\left\| y \right\| -\min_{a\in \Delta_d}\left\| a\circ y' \right\| =\left\| y \right\| -\left\| \tilde{y} \right\| +\left\| \tilde{y} \right\| -\min_{a\in \Delta_d}\left\| a\odot y' \right\| \\&\qquad \qquad +\min_{a\in \Delta_d}\left\| a\odot \tilde{y}' \right\| -\min_{a\in \Delta_d}\left\| a\odot \tilde{y}' \right\| \\
&\leqslant \left\| y-y' \right\| +\min_{a\in \Delta_d}\left\| a\odot \tilde{y}' \right\| -\min_{a\in \Delta_d}\left\| a\odot y' \right\|_{}\\&=\left\| y'-y \right\| +\max_{a\in \Delta_d}\min_{a'\in \Delta_d}\left\{ \left\| a'\odot \tilde{y}' \right\| -\left\| a\odot y' \right\|  \right\}\\
&\leqslant \left\| y-y' \right\| +\max_{a\in \Delta_d} \left\| a\odot (\tilde{y}'-y') \right\|   = \left\| (y,y')-(\tilde{y},\tilde{y}') \right\|_{\mathcal{V}^*},
\end{align*}
where the first inequality follows from the reverse triangle inequality and the definition of $\mathcal{C}$ and the third inequality from removing the minimum over $a'\in \Delta_d$ and using the reverse triangle inequality again.
Then, taking the minimum over $(\tilde{y},\tilde{y}')\in \mathcal{C}$ and applying Proposition~\ref{prop:distance-support} gives the result:
\[ \bar{\operatorname{Reg}}_T\leqslant \min_{(\tilde{y},\tilde{y}')\in \mathcal{C}}\left\| (y,y')-(\tilde{y},\tilde{y}') \right\|_{\mathcal{V}^*}=I^*_{\mathcal{B}\cap \mathcal{C}^\circ }(y,y')=I_{\mathcal{B}\cap \mathcal{C}^\circ }^*\left( \frac{1}{T}\sum_{t=1}^Tr(a_t,\ell_t) \right). \]
\end{proof}

\subsection{An algorithm for $\ell_p$ global cost functions}
\label{sec:an-algorithm-ell_p}
We define and analyze an algorithm based on a carefully chosen regularizer
which takes advantage of the properties of $\ell_p$ norms. The
construction for general norms in given in Appendix~\ref{sec:algo-global-cost}.
We consider on
$\mathcal{X}=\mathcal{B}\cap \mathcal{C}^\circ$ the following regularizer:
\[ h(z,z')=
  \begin{cases}
    \displaystyle \frac{A}{2}\left\| z \right\|_2^2+\frac{1}{2}\left\| z' \right\|_{q'}^2&\text{if
      $(z,z')\in \mathcal{B}\cap \mathcal{C}^\circ$},\\
    +\infty&\text{otherwise}.
  \end{cases}
\]
where $q'\in (1,2]$ and $A>0$ are to be chosen later. The
algorithm associated with a positive sequence
$(\eta_t)_{t\geqslant 1}$ and an oracle $a$ from
Remark~\ref{rk:oracle} writes, for $t\geqslant 1$,
\begin{align*}
\text{compute}\quad x_t&=\argmax_{x\in \mathcal{X}}\left\{ \left< \eta_{t-1}\sum_{s=1}^{t-1}r_s, x \right> -h(x) \right\}\\
\text{compute}\quad a_t&=\argmin_{a\in
                         \Delta_d}\sum_{i=1}^d\max_{}(0,z_{ti}a_i+z'_{ti}),\quad
                         \text{where $(z_t,z_t')=x_t$,}\\
\text{observe}\quad r_t&:=r(a_t,b_t).
\end{align*}
\begin{theorem}
\label{thm:global_costs_lp_norms}
Let $p\in (1,+\infty]$ and assume $\left\|\,\cdot\,\right\| =\left\|\,\cdot\,\right\|_p$. 
Then, the above algorithm with $A=\min_{}\left\{ d^{1-2/p},1 \right\}$,
$q'=1+(2\log d-1)^{-1}$ and coefficients
\[ \eta_t= \frac{1}{2\sqrt{t\max_{}\left\{ d^{2/p-1},\ e(2\log d-1) \right\}}},\quad t\geqslant 1, \]
guarantees, against any sequence $(\ell_t)_{t\geqslant 1}$ in $[0,1]^d$ chosen by the Environment,
\begin{equation}
\label{eq:lp-upper-bound}
\forall T\geqslant 1,\quad \bar{\operatorname{Reg}}_T \leqslant \frac{4}{\sqrt{T}}\max_{}\left\{ d^{1/p-1/2},\ \sqrt{2e\log d} \right\}.
\end{equation}
\end{theorem}
\begin{remark}
  In the special case $p=\infty$, the above bound recovers the best
  known bound of order
  $O(\sqrt{(\log d)/T})$ from \cite{rakhlin2011online}.
  For $1<p<+\infty$, we obtain, to the best of our knowledge, the
  first bounds with explicit dependence in $d$ and $p$.  Surprisingly,
  the same $O(\sqrt{(\log d)/T})$ bound with logarithmic dependence in
  the dimension $d$ also holds for all $p\geqslant 2$. We were unable
  to find in the literature any lower bound for a given cost
  function\footnote{\cite{even2009online} gives a lower bound, but is
    of a different kind, as the cost
    function depends on the time horizon.}, and standard techniques
  from regret minimization, which involves a randomized Environment
 which cancels the influence of the Decision Maker on its own reward, do not seem to work at all, because of the particular form of the quantity to be minimized. Developing lower bound techniques for this kind of online learning problems appears to be an interesting and challenging research direction.
\end{remark}
\begin{proof}
  We aim at applying Theorem~\ref{thm:approachability-mirror-descent}.
Let us first establish an upper bound on the difference between the
highest and lowest values of $h$.
Note that
$\max_{a\in \Delta_d}\left\|a\odot y'\right\|_{p}=\left\| y' \right\|_{\infty}$
for all $y'\in \mathbb{R}^d$. Indeed, by denoting $e_1,\dots,e_d$ the
canonical basis of $\mathbb{R}^d$, and using the fact that $\left\|\,\cdot\,\right\|_p\leqslant \left\|\,\cdot\,\right\|_1$,
\begin{align*}
  \left\| y' \right\|_{\infty}&=\max_{1\leqslant i\leqslant d}\left| y'_i \right| =\max_{a\in \left\{ e_1,\dots,e_d \right\}}\left\| a\odot y' \right\|_p\leqslant \max_{a\in \Delta_d}\left\| a\odot y' \right\|_p \\
  &\leqslant \max_{a\in \Delta_d}\left\| a\odot y' \right\|_1=\max_{a\in \Delta_d}\sum_{i=1}^da_i\left| y_i' \right| =\left\| y' \right\|_{\infty}.
\end{align*}
Therefore,
$\left\|(y,y')\right\|_{\mathcal{V}^*}=\left\| y \right\|_p +\left\| y' \right\|_{\infty}$
for all $(y,y')\in \mathcal{V}^*$. Using a standard argument, we can
prove that its dual norm writes
\[ \left\| (z,z') \right\|_{\mathcal{V}}=\max_{}\left\{ \left\| z \right\|_q,\ \left\| z' \right\|_{1} \right\},\qquad (z,z')\in \mathcal{V}, \]
where $q=(1-1/p)^{-1}$.
Therefore,
$\mathcal{B}=\left\{ (z,z')\in \mathcal{V},\ \left\| z \right\|_q\leqslant 1\text{
  and }\left\| z' \right\|_1\leqslant 1 \right\}$.
Besides, because $0\in \mathcal{X}$, it
holds that $\min_{\mathcal{X}}h=0$. Therefore, using the standard
inequality between $\ell_p$ norms that can be written
$\left\|\,\cdot\,\right\|_{q'}\leqslant d^{\max_{}(1/q'-1/q,0)}\left\|\,\cdot\,\right\|_q$, 
\begin{align}
  \label{eq:delta}
  \begin{split}
  \max_{\mathcal{X}}h-\min_{\mathcal{X}}h&\leqslant \max_{(z,z')\in \mathcal{B}}h(z,z')=\max_{\substack{\left\| z \right\|_q\leqslant 1\\\left\| z' \right\|_{1}\leqslant 1}}\left\{ \frac{A}{2}\left\| z \right\|_2^2+\frac{1}{2}\left\| z' \right\|_{q'}^2 \right\}\\
                                         &\leqslant \max_{\substack{\left\| z \right\|_q\leqslant 1\\\left\| z' \right\|_{1}\leqslant 1}}\left\{ \frac{A}{2}d^{\max (1-2/q,0)}\left\| z \right\|_q^2+\frac{1}{2}\left\| z' \right\|_1^2 \right\}\\
  &=\frac{1}{2}(A\, d^{\max (1-2/q,0)}+1)=\frac{1}{2}\left( A\, d^{\max (2/p-1,0)}+1 \right).
  \end{split}
\end{align}

  Let us introduce the following norm on the payoff space
  $\mathcal{V}^*$, which is different from the norm
  $\left\|\,\cdot\,\right\|_{\mathcal{V}^*}$ involved in Proposition~\ref{prop:global-cost-support-function}:
  \[ \left\| (y,y') \right\|_{(\mathcal{V}^*)}=\left\| y \right\| _1+\left\| y' \right\|_{\infty}. \]
We can see that the vector-valued payoffs are bounded by $2$ with respect to
this norm. Indeed, for all $a\in \Delta_d$ and $\ell\in [0,1]^d$,
\[ \left\| r(a,\ell) \right\|_{(\mathcal{V}^*)}=\left\| a\odot \ell \right\|_1+\left\| \ell \right\|_{\infty}\leqslant 2.  \]
Denote $\left\|\,\cdot\,\right\|_{(\mathcal{V})}$ the dual norm of
$\left\|\,\cdot\,\right\|_{(\mathcal{V}^*)}$, which has the following
expression:
$\left\| (z,z') \right\|_{(\mathcal{V})}=\max_{}(\left\| z \right\|_{\infty},\left\| z' \right\|_{1})$
for all $(z,z')\in \mathcal{V}$.

Let us now prove for regularizer $h$ a strong
convexity property with respect to
$\left\|\,\cdot\,\right\|_{\mathcal{V}}$.
It can be practical to write $h$ as
\[ h(z,z')=h_1(z)+h_2(z')+I_{\mathcal{X}}(z,z'),\qquad (z,z')\in \mathcal{V}, \]
where $h_1(z)=\frac{A}{2}\left\| z \right\|_2^2$ and $h_2(z)=\frac{1}{2}\left\| z' \right\|_{q'}^2$.
We note that according to Proposition~\ref{prop:lp-regularizer}, $h_1$ is $A$-strongly convex with respect to
$\left\|\,\cdot\,\right\|_2$ and $h_2$ is $(q'-1)d^{2(1/q'-1)}$-strongly convex with
respect to $\left\|\,\cdot\,\right\|_1$.
For all
$(z,z'),(\tilde{z},\tilde{z}')\in \mathcal{V}$, and
$\lambda\in [0,1]$, denote
$z_{\lambda}= \lambda z +(1-\lambda)\tilde{z}$ and
$z_{\lambda}'= \lambda z' +(1-\lambda)\tilde{z}'$. Then, using the
strong convexity properties of $h_1$ and $h_2$, and the fact that $\left\|\,\cdot\,\right\|_2\geqslant \left\|\,\cdot\,\right\|_{\infty}$,
\begin{align*}
  \lambda h(z,z')+(1-\lambda)h(\tilde{z},\tilde{z}')&\geqslant \lambda(h_1(z)+h_2(z'))+(1-\lambda)(h_1(\tilde{z})+h_2(\tilde{z}'))\\
                                                    &=\lambda h_1(z)+(1-\lambda)h_1(\tilde{z})+ \lambda h_2(\tilde{z})+(1-\lambda)h_2(\tilde{z}')\\
                                                    &\geqslant h_1(z_{\lambda})+\frac{A \lambda(1-\lambda)}{2}\left\| \tilde{z}-z \right\|_2^2+h_2(z_{\lambda}')\\&\qquad\qquad  +\frac{(q'-1)d^{2(1/q'-1)}\lambda(1-\lambda)}{2}\left\| \tilde{z}'-z' \right\|_1^2\\
  &\geqslant h(z_{\lambda},z_{\lambda}')\\&\qquad +\min_{}\left\{  A,\ (q'-1)d^{2(1/q'-1)} \right\} \frac{\lambda(1-\lambda)}{2}\left\| (\tilde{z},\tilde{z}')-(z,z') \right\|_{(\mathcal{V})}^2.
\end{align*}
Therefore, $h$ is
$\min_{}\left\{  A,\ (q'-1)d^{2(1/q'-1)} \right\}$-strongly convex with
respect to $\left\|\,\cdot\,\right\|_{(\mathcal{V})}$.

Applying Theorem~\ref{thm:approachability-mirror-descent} with
\[ M=2,\quad \Delta=\frac{1}{2}(A\, d^{\max (2/p-1,0)}+1),\quad \text{and}\quad K=\min_{}\left\{ A,\ (q'-1)d^{2/q'-2} \right\}, \]
together with Proposition~\ref{prop:global-cost-support-function} gives
\[ \bar{\operatorname{Reg}}_T\leqslant \frac{4}{\sqrt{2}}\sqrt{\frac{A\, d^{\max (2/p-1,0)}+1}{T\,\min_{}\left\{ A,\ (q'-1)d^{2/q'-2} \right\}}}=\frac{4}{\sqrt{T}}\max_{}\left\{ d^{1/p-1/2},\ \sqrt{e(2\log d-1)} \right\}, \]
where the equality follows from the choice $A=\min_{}\left\{ d^{1-2/p},1 \right\}$ and
$q'=1+(2\log d-1)^{-1}$. Hence the result.
\end{proof}

%% file: acknowledgements.tex
The author is grateful to Rida Laraki, Vianney Perchet, Sylvain Sorin
and Gilles Stoltz for inspiring discussions, insightful advice and
careful proofreading. The paper also benefited from the numerous
suggestions of improvements from the two anonymous referees.  This
work was supported by a public grant as part of the
\emph{Investissement d'avenir} project, reference
\texttt{ANR-11-LABX-0056-LMH}, LabEx LMHA.

%% file: appendix.tex
\section{Definitions and properties about closed convex cones}
We recall the definitions of a closed convex cone, of the polar cone, and
gather a few properties. \(\mathcal{W}\) will be a finite-dimensional vector
space and \(\mathcal{W}^*\) its dual.

\begin{definition}
\label{def:closed-convex-cone}
A nonempty subset $\mathcal{C}$ of $\mathcal{W}$ is a \emph{closed convex cone} if it is closed and if
for all $y,y'\in \mathcal{C}$ and $\lambda\in \mathbb{R}_+$, we have $y+y'\in \mathcal{C}$ and $\lambda y\in \mathcal{C}$.
\end{definition}

\begin{definition}
\label{def:polar-cone}
Let $\mathcal{A}$ be a subset of $\mathcal{W}$. The \emph{polar cone} of $\mathcal{A}$ is a subset of the dual space $\mathcal{W}^*$ defined by
\[ \mathcal{A}^\circ =\left\{ x\in \mathcal{W}^*\,,\, \forall y\in \mathcal{A},\ \left< y , x \right> \leqslant 0 \right\}.  \]
\end{definition}
\begin{figure}
  \centering
\begin{tikzpicture}[y=0.80pt, x=0.80pt, yscale=-1.000000, xscale=1.000000, inner sep=0pt, outer sep=0pt,scale=2]
  \path[draw=black,fill=ccccccc,line join=miter,line cap=butt,even odd rule,line
    width=0.800pt] (0.0000,987.3622) -- (0.0000,1052.3622) -- (50.0000,1002.3622);
  \path[draw=black,fill=ce6e6e6,line join=miter,line cap=butt,even odd rule,line
    width=0.800pt] (-55.0000,1052.3622) -- (0.0000,1052.3622) --
    (35.0000,1087.3622);
  \path[draw=black,fill=c999999,line join=miter,line cap=butt,even odd rule,line
    width=0.800pt] (0.0000,1017.3622) .. controls (0.0000,1024.8622) and
    (6.4246,1021.0350) .. (11.5089,1024.0855) .. controls (15.8839,1026.7106) and
    (18.7500,1033.6122) .. (25.0000,1027.3622) .. controls (31.2500,1021.1122) and
    (23.5102,1014.0792) .. (16.6352,1010.9542) .. controls (9.7602,1007.8292) and
    (0.0000,1009.8622) .. (0.0000,1017.3622) -- cycle;
    \draw (0.0000,1052.3622) node {$\bullet$} node[above left=5pt]
    {$0$} node[below=15pt] {$\mathcal{A}^\circ$};
    \draw (30.0000,1002.3622) node {$\mathcal{A}^{\circ \circ}$};
    \draw (15.0000,1018.3622) node {$\mathcal{A}$};
\end{tikzpicture}
  \caption{The polar cone of a set $\mathcal{A}$ and the bipolar}
  \label{fig:polar}
\end{figure}
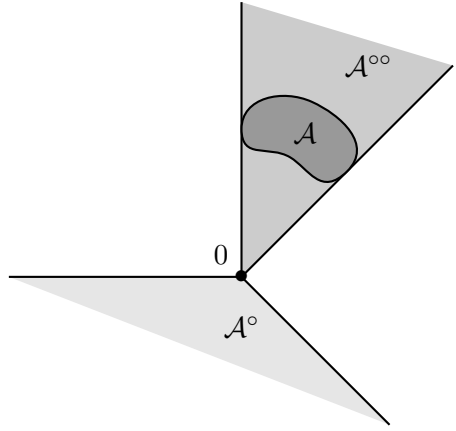

\label{sec:prop-clos-conv}
The following proposition is an immediate consequence of the bipolar theorem---see e.g.Theorem 3.3.14 in~\cite{borwein2010convex}.
\begin{proposition}
\label{prop:polar-smallest}
Let $\mathcal{A}$ be a subset of $\mathcal{W}$.
\begin{enumerate}[(i)]
\item $\mathcal{A}^{\circ \circ }$ is the smallest closed convex cone containing $\mathcal{A}$.
\item If $\mathcal{A}$ is closed and convex, then $\mathcal{A}^{\circ \circ }=\mathbb{R}_+\mathcal{A}$.
\item If $\mathcal{A}$ is a closed convex cone, then $\mathcal{A}^{\circ \circ }=\mathcal{A}$.
\end{enumerate}
\end{proposition}

The following statement is a simpler version of Moreau's decomposition
theorem~\citep{moreau1962decomposition}.
\begin{proposition}
\label{prop:polar-euclidean}
Assume that $\mathcal{W}$ is an Euclidean space. We identify $\mathcal{W}$ and
its dual space $\mathcal{W}^*$. Let $\mathcal{C}$ be a closed convex cone in $\mathcal{W}$, and $y\in \mathcal{W}$. Then, $y-\operatorname{proj}_{\mathcal{C}}y=\operatorname{proj}_{\mathcal{C}^\circ}y$, where $\operatorname{proj}$ denotes the Euclidean projection. In particular, $y-\operatorname{proj}_{\mathcal{C}}y$ belongs to $\mathcal{C}^\circ$.
\end{proposition}
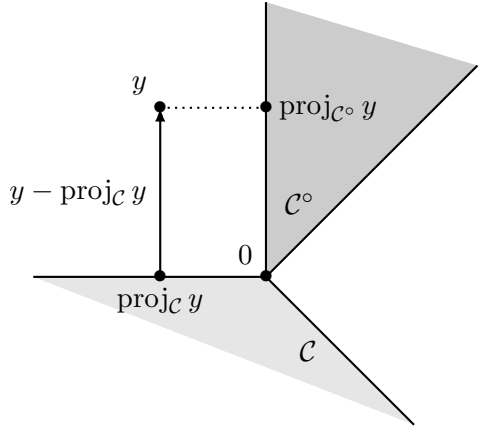
\begin{figure}[h!]
  \centering
\begin{tikzpicture}[y=0.80pt, x=0.80pt, yscale=-1.000000, xscale=1.000000, inner sep=0pt, outer sep=0pt,scale=2]
  \path[draw=black,fill=ccccccc,line join=miter,line cap=butt,even odd rule,line
    width=0.800pt] (0.0000,987.3622) -- (0.0000,1052.3622) -- (50.0000,1002.3622);
  \path[draw=black,fill=ce6e6e6,line join=miter,line cap=butt,even odd rule,line
    width=0.800pt] (-55.0000,1052.3622) -- (0.0000,1052.3622) --
    (35.0000,1087.3622);
  \path[draw=black,line join=miter,line cap=butt,even odd rule,line width=0.800pt,->,>=latex]
    (-25.0000,1052.3622) -- (-25.0000,1012.3622) node[midway,left=5pt]
    {$y-\proj_{\mathcal{C}}y$};
  \path[draw=black,line join=miter,line cap=butt,even odd rule,line width=0.800pt,dotted]
    (0.0000,1012.3622) -- (-25.0000,1012.3622);
    \draw(0.0000,1052.3622) node {$\bullet$} node[above left=5pt] {$0$};
    \draw (-25.0000,1012.3622) node {$\bullet$} node[above left=5pt] {$y$};
    \draw (-25.0000,1052.3622) node {$\bullet$} node[below=4pt] {$\proj_{\mathcal{C}}y$};
    \draw (0.0000,1012.3622) node {$\bullet$} node[right=5pt] {$\proj_{\mathcal{C}^\circ }y$};
    \draw (10,1070) node {$\mathcal{C}$};
    \draw (8,1035) node {$\mathcal{C}^\circ$};
\end{tikzpicture}
  \caption{Illustration of Proposition~\ref{prop:polar-euclidean}}
  \label{fig:polar-euclidean}
\end{figure}

\subsection{Proof of Proposition~\ref{prop:generator-examples}}
\label{sec:proof-generator-examples}

(\ref{item:orthant}) is easy.
(\ref{item:generator-always}) holds
because $\mathcal{B}\cap \mathcal{C}$
is indeed nonempty, convex as the intersection of two convex sets, and
for any point $x\in \mathcal{C}\setminus \left\{ 0 \right\}$, $x/\left\| x \right\|_{}$ belongs to
$\mathcal{B}\cap \mathcal{C}$, so that $\mathbb{R}_+(\mathcal{B}\cap \mathcal{C})=\mathcal{C}$.
(\ref{item:Z-generator-of-what}) is a consequence of Proposition~\ref{prop:polar-smallest}.

\section{Properties of regularizers}
\label{sec:proofs-prop-regul}

\begin{proposition}
\label{prop:lengendre}
Let $h$ be a regularizer on $\mathcal{X}$. Its Legendre--Fenchel transform, defined by
\[ h^*(y)=\sup_{x\in \mathcal{V}}\left\{ \left< y , x \right> -h(x) \right\},\quad y\in \mathcal{V}^*,  \]
satisfies the following properties.
\begin{enumerate}[(i)]
\item\label{item:legendre-1} $\operatorname{dom}h^*=\mathcal{V}^*$;
\item\label{item:legendre-2} $h^*$ is differentiable on $\mathcal{V}^*$;
\item\label{item:legendre-3} For all $y\in \mathcal{V}^*$, $\nabla h^*(y)=\argmax_{x\in \mathcal{X}}\left\{ \left< y , x \right> -h(x) \right\}$. In particular, $\nabla h^*$ takes values in $\mathcal{X}$.
\end{enumerate}
\end{proposition}
\begin{proof}
(\ref{item:legendre-1}) Let $w\in \mathcal{V}^*$. The function $x\longmapsto \left< w , x \right> -h(x)$ equals $-\infty$ outside of $\mathcal{X}$, and is upper semicontinuous on
$\mathcal{X}$ which is compact. It thus has a maximum and $h^*(w)<+\infty$.

(\ref{item:legendre-2},\ref{item:legendre-3}) Moreover, this maximum is attained at
a unique point because $h$ is strictly convex. Besides, for $x\in \mathcal{V}$ and $w\in \mathcal{V}^*$
 \[ x\in \partial h^*(w)\quad  \Longleftrightarrow \quad w\in \partial h(x)\quad \Longleftrightarrow \quad x\in \argmax_{x'\in \mathcal{X}}\left\{ \left< w , x' \right> -h(x') \right\},  \]
in other words, $\partial h^*(w)=\argmax_{x'\in \mathcal{X}}\left\{ \left< w , x' \right> -h(x') \right\}$. This argmax is a singleton as we noticed.
It means that $h^*$ is differentiable.
\end{proof}

Recall that \(\Delta_d\) denotes the unit simplex of \(\mathbb{R}^d\): \(\Delta_d=\left\{ x\in \mathbb{R}_+^d\,\middle|\,\sum_{i=1}^dx_i=1 \right\}\).
\begin{definition}[Entropic regularizer]
The \emph{entropic regularizer} $h_{\mathrm{ent}}:\mathbb{R}^d\to
\mathbb{R}\cup\{+\infty\}$ is defined as
\[ h_{\mathrm{ent}}(x)=\begin{cases} \sum_{i=1}^dx_i\log x_i&\text{if $x\in \Delta_d$}\\
+\infty&\mathrm{otherwise}, \end{cases} \]
where $x_i\log x_i=0$ when $x_i=0$.
\end{definition}

\begin{proposition}
\label{prop:entropy}
\begin{enumerate}[(i)]
\item\label{item:entropy-regularizer} $h_{\mathrm{ent}}$ is a regularizer on $\Delta_d$;
\item\label{item:entropy-map} $\displaystyle \nabla h^*_{\mathrm{ent}}(y)=\left( \frac{\exp y_i}{\sum_{j=1}^d\exp y_j} \right)_{1\leqslant j\leqslant d}$, for all $y\in \mathbb{R}^d$;
\item\label{item:entropy-delta} $\max_{x\in \Delta_d}h_{\mathrm{ent}}(x)-\min_{x\in \Delta_d}h_{\mathrm{ent}}(x)=\log d$;
\item\label{item:entropy-strong-convexity} $h_{\mathrm{ent}}$ is $1$-strongly convex with respect to $\left\|\,\cdot\,\right\|_1$. 
\end{enumerate}
\end{proposition}
\begin{proof}
(\ref{item:entropy-regularizer}) is immediate, and (\ref{item:entropy-map}) is classic---see e.g.~\cite[Example 2.25]{boyd2004convex}.

(\ref{item:entropy-delta}) $h_{\mathrm{ent}}$ being convex, its maximum on $\Delta_d$ is attained at one of the extreme points.
At each extreme point, the value of $h_{\mathrm{ent}}$ is
zero. Therefore, $\max_{\Delta_d}h_{\mathrm{ent}}=0$. As for the
minimum, $h_{\mathrm{ent}}$ being convex and symmetric with respect to
the components $x_i$, its minimum is attained
at the centroid $(1/d,\dots,1/d)$ of the simplex $\Delta_d$, where its value is $-\log d$. Therefore, $\min_{\Delta_d}h_{\mathrm{ent}}=-\log d$ and $\max_{\Delta_d} h_{\mathrm{ent}}-\min_{\Delta_d}h_{\mathrm{ent}}=\log d$.

(\ref{item:entropy-strong-convexity}) Consider $F:\mathbb{R}^d\to \mathbb{R}\cup\{+\infty\}$ defined by
\[ F(x)= \begin{cases}\sum_{i=1}^d(x_i\log x_i-x_i)+1&\text{if $x\in \mathbb{R}_+^d$}\\
+\infty&\text{otherwise}. \end{cases}
 \]
 Let us prove that $F$ is $1$-strongly convex with respect to $\left\|\,\cdot\,\right\|_{1}$.
By definition, the domain of $F$ is $\mathbb{R}_+^d$. It is differentiable on the interior of the domain $(\mathbb{R}_+^*)^d$ and $\nabla F(x)=(\log x_i)_{1\leqslant i\leqslant d}$ for $x\in (\mathbb{R}_+^*)^d$.
Therefore, the norm of $\nabla F(x)$ goes to $+\infty$ when $x$ converges to a boundary point of $\mathbb{R}_+^d$.~\cite[Theorem 26.1]{rockafellar1970convex} then assures that the subdifferential $\partial F(x)$ is empty as soon as $x\not \in (\mathbb{R}_+^*)^d$.
Therefore, the characterization of strong convexity from~\cite[Lemma 14]{shalev2007online}, which we aim at proving, can be written
\begin{equation}
\label{eq:6}
\left< \nabla F(x')-\nabla F(x) , x'-x \right> \geqslant \left\| x'-x \right\|_1^2,\quad x,x'\in (\mathbb{R}_+^*)^d.
\end{equation}
Let $x,x'\in (\mathbb{R}_+^*)^d$.
\[ \left< \nabla F(x')-\nabla F(x) , x'-x \right>=\sum_{i=1}^d\log \frac{x'_i}{x_i}(x'_i-x_i). \]
A simple study of function shows that $(s-1)\log s-2(s-1)^2/(s+1)\geqslant 0$ for $s\geqslant 0$. Applied with $s=x'_i/x_i$, this gives
\[ \sum_{i=1}^d\log \frac{x'_i}{x_i}(x'_i-x_i)\geqslant \left\| x'-x \right\|_1^2, \]
and (\ref{eq:6}) is proved. $F$ is therefore $1$-strongly convex with respect to $\left\|\,\cdot\,\right\|_1$ and so is $h_{\mathrm{ent}}$ thanks to Lemma~\ref{lm:strong-convexity-restriction}.
\end{proof}

\begin{definition}[$\ell_p$ regularizer]
For $p\in (1,2]$ and a nonempty convex compact subset
$\mathcal{X}$ of $\mathbb{R}^d$, the associated \emph{$\ell_p$
  regularizer} is defined as
\[ h_{p}(x)=\begin{cases} \frac{1}{2}\left\| x \right\|_p^2&\text{if $x\in \mathcal{X}$}\\
+\infty&\mathrm{otherwise}. \end{cases} \]
\end{definition}

\begin{proposition}
\label{prop:lp-regularizer}
Let $p\in (1,2]$.
\begin{enumerate}[(i)]
\item\label{item:lp-regularizer} $h_p$ is a regularizer on $\mathcal{X}$;
\item\label{item:lp-strong-convexity} $h_p$ is $(p-1)d^{2(1/p-1)}$-strongly convex with respect to $\left\|\,\cdot\,\right\|_1$;
\item\label{item:euclidean-strongly-convex} $h_2$ is $1$-strongly
  convex with respect to $\left\|\,\cdot\,\right\|_2$;
\item\label{item:euclidean-projection} $\nabla h^*_2(y)=\proj_{\mathcal{X}}(y)$ for all $y\in \mathbb{R}^d$ where
$\proj_{\mathcal{X}}$ denotes the Euclidean projection onto $\mathcal{X}$.
\end{enumerate}
\end{proposition}
\begin{proof}
(\ref{item:lp-regularizer}) Since $p\geqslant 1$, $\left\|\,\cdot\,\right\|_p$ is a norm and is therefore convex. $h_p$ then clearly is a regularizer on $\mathcal{X}$.
(\ref{item:lp-strong-convexity},\ref{item:euclidean-strongly-convex})
We consider the function $F(x)=\frac{1}{2}\left\| x \right\|_p^2$
defined on $\mathbb{R}^d$ which is $(p-1)$-strongly convex with
respect to $\left\|\,\cdot\,\right\|_p$---see
e.g.~\cite{bubeck2011introduction} or \cite[Corollary
10]{kakade2012regularization}. Then, so is $h_p$ thanks to
Lemma~\ref{lm:strong-convexity-restriction}. Substituting $p=2$ gives
(\ref{item:euclidean-strongly-convex}). The strong convexity with
respect to $\left\|\,\cdot\,\right\|_1$ follows from the standard
comparison $\left\|\,\cdot\,\right\|_{p}\geqslant
d^{1/q-1}\left\|\,\cdot\,\right\|_1$ in
$\mathbb{R}^d$. (\ref{item:euclidean-projection}) For all $y\in
\mathbb{R}^d$, using property (\ref{item:legendre-3}) from
Proposition~\ref{prop:lengendre},
\begin{align*}
\nabla h_2^*(y)&=\argmax_{x\in \mathcal{X}}\left\{ \left< y , x
                 \right> -\frac{1}{2}\left\| x \right\|_2^2
                 \right\}=\argmin_{x\in \mathcal{X}}\left\{
                 \frac{1}{2}\left\| x \right\|_2^2-\left< y , x
                 \right> +\frac{1}{2}\left\| y \right\|_2^2 \right\}\\
  &=\argmin_{x\in \mathcal{X}}\left\| y-x \right\|_2^2=\operatorname{proj}_{\mathcal{X}}(y).
\end{align*}
\end{proof}

\begin{lemma}
\label{lm:strong-convexity-restriction}
Let $\left\|\,\cdot\,\right\|_{}$ a norm on $\mathcal{V}$, $K>0$ and $h,F:\mathcal{V}\to \mathbb{R}\cup\{+\infty\}$ two convex functions such that for all $x\in \mathcal{V}$,
\[ h(x)=F(x)\quad \text{or}\quad h(x)=+\infty. \]
Then, if $F$ is $K$-strongly convex with respect to $\left\|\,\cdot\,\right\|_{}$, so is $h$.
\end{lemma}
\begin{proof}
Note that for all $x\in \mathcal{V}$, $F(x)\leqslant h(x)$.
Let us prove that $h$ satisfies the condition from Definition~\ref{def:strong-convexity}. Let $x,x'\in \mathcal{V}$, $\lambda\in [0,1]$ and denote $x''=\lambda x+(1-\lambda)x'$.
Let us first assume that $h(x'')=+\infty$. By convexity of $h$, either $h(x)$ or $h(x')$ is equal to $+\infty$, and the right-hand side of (\ref{eq:definition-strong-convexity}) is equal to $+\infty$. Inequality (\ref{eq:definition-strong-convexity}) therefore holds.
If $h(x'')$ is finite,
\begin{align*}
h(x'')&=F(x'')\leqslant \lambda F(x)+(1-\lambda)F(x')-\frac{K\lambda(1-\lambda)}{2}\left\| x'-x \right\|_{}^2\\ 
&\leqslant \lambda h(x)+(1-\lambda)h(x')-\frac{K\lambda(1-\lambda)}{2}\left\| x'-x \right\|_{}^2,
\end{align*}
and (\ref{eq:definition-strong-convexity}) is proved.
\end{proof}

\section{Various postponed proofs}
\label{sec:postponed-proofs}

\subsection{Proof of Proposition~\ref{prop:distance-support}}
\label{sec:proof-support}
Let $y\in \mathcal{V}^*$. Using the definition of the dual norm and Sion's minimax theorem,
\[ \inf_{y'\in \mathcal{C}}\left\| y'-y \right\|_{*}=\inf_{y'\in \mathcal{C}}\sup_{x\in \mathcal{B}}\left< y-y' , x \right> =\sup_{x\in \mathcal{B}}\inf_{y'\in \mathcal{C}}\left\{ \left< y , x \right> -\left< y' , x \right>  \right\}.  \]
Suppose $x$ does not belong to $\mathcal{C}^{\circ }$. Then, there
exists $y_0'\in \mathcal{C}$ such that
$\left< y'_0 , x \right>>0$. $\mathcal{C}$ being closed by
multiplication by $\mathbb{R}_+$, the quantity
$\left< y' , x \right>$ (with $y'\in \mathcal{C}$) can be made
arbitrarily large by selecting $y'=\lambda y_0'$ and letting $\lambda\to +\infty$, and thus the above infimum is equal to $ -\infty$.
Therefore, we can restrict the above supremum to $\mathcal{B}\cap \mathcal{C}^{\circ }$. We thus have
\[ \inf_{y'\in \mathcal{C}}\left\| y'-y \right\|_{*}=\sup_{x\in \mathcal{B}\cap \mathcal{C}^\circ }\left\{ \left< y , x \right> -\sup_{y'\in \mathcal{C}}\left< y' , x \right>  \right\}.  \]
The above embedded supremum is zero because for $x\in \mathcal{B}\cap \mathcal{C}^\circ$ and $y'\in \mathcal{C}$ we obviously have $\left< y' , x \right> \leqslant 0$, and $0$ is attained with $y'=0$. Finally,
\[ \inf_{y'\in \mathcal{C}}\left\| y'-y \right\|_{*}=\sup_{x\in \mathcal{B}\cap \mathcal{C}^\circ }\left< y , x \right> =I_{\mathcal{B}\cap \mathcal{C}^\circ }^*(y). \]

\subsection{Proof of Proposition~\ref{prop:dual-condition}}
\label{sec:proof-dual-condition}
Blackwell's condition can be written
\[ \max_{x\in \mathcal{C}^\circ }\min_{a\in \mathcal{A}} \max_{b\in \mathcal{B}}\left< r(a,b) ,x \right> \leqslant 0. \]
Since the above dot product is affine in each of the variables $a$, $b$
and $x$, by applying Sion's minimax theorem twice, the above is
equivalent to
\[ \max_{b\in \mathcal{B}}\min_{a\in \mathcal{A}} \max_{x\in \mathcal{C}^\circ } \left< r(a,b) ,x \right> \leqslant 0, \]
which is exactly the dual condition.

\subsection{Proof of Lemma~\ref{lm:regret-bound}}
\label{sec:proof-regret-bound}
Assume that the sequence of parameters $(\eta_t)_{t\geqslant 1}$ is nonincreasing.
Denote $Y_t=\sum_{s=1}^tr_t$ for $t\geqslant 1$ and $\eta_0=\eta_1$. Let $x\in \mathcal{X}$. Using Fenchel's inequality, we write
\begin{align}
\label{eq:3}
\begin{split}
\left< Y_T , x \right> &=\frac{\left< \eta_T Y_T, x \right> }{\eta_T}\leqslant \frac{h^*(\eta_TY_T)}{\eta_T}+\frac{h(x)}{\eta_T}\\
&\leqslant \frac{h^*(0)}{\eta_0}+\sum_{t=1}^T\left( \frac{h^*(\eta_tY_t)}{\eta_t}-\frac{h^*(\eta_{t-1}Y_{t-1})}{\eta_{t-1}} \right)+\frac{\max_{x\in \mathcal{X}}h(x)}{\eta_T}.
\end{split}
\end{align}
Let us bound $h^*(\eta_tY_t)/\eta_t$ from above. For all $x\in \mathcal{X}$ we have
\[ \frac{\left< \eta_tY_t , x \right> -h(x)}{\eta_t}=\frac{\left< \eta_{t-1}Y_t , x \right> -h(x)}{\eta_{t-1}}-h(x)\left( \frac{1}{\eta_t}-\frac{1}{\eta_{t-1}} \right).  \]
The maximum over $x\in \mathcal{X}$ of the above left-hand side gives $h^*(\eta_tY_t)/\eta_t$.
As for the right-hand side, let us take the maximum over $x\in
\mathcal{X}$ for each of the two terms separately. This gives
\begin{align*}
\frac{h^*(\eta_tY_t)}{\eta_t}&\leqslant \max_{x\in \mathcal{X}}\left\{  \frac{\left< \eta_{t-1}Y_t , x \right> -h(x)}{\eta_{t-1}}\right\}+\max_{x\in \mathcal{X}}\left\{ -h(x)\left( \frac{1}{\eta_t}-\frac{1}{\eta_{t-1}} \right)  \right\}\\
&=\frac{h^*(\eta_{t-1}Y_t)}{\eta_{t-1}}+\left( \min_{x\in \mathcal{X}}h(x) \right)\left( \frac{1}{\eta_{t-1}}-\frac{1}{\eta_t} \right),  
\end{align*}
where we used the fact that the sequence $(\eta_t)_{t\geqslant 0}$ is nonincreasing. Injecting this inequality in (\ref{eq:3}), we get
\begin{multline*}
 \left< Y_T , x \right> \leqslant \frac{h^*(0)}{\eta_0}+\sum_{t=1}^T\frac{h^*(\eta_{t-1}Y_t)-h^*(\eta_{t-1}Y_{t-1})}{\eta_{t-1}}\\+\left( \min_{x\in \mathcal{X}} h(x)\right)\sum_{t=1}^T\left( \frac{1}{\eta_{t-1}}-\frac{1}{\eta_t} \right)+\frac{\max_{x\in \mathcal{X}}h(x)}{\eta_T}. 
\end{multline*}
We now make the quantity 
\[  D_{h^*}(\eta_{t-1}Y_t,\eta_{t-1}Y_{t-1}):=h^*(\eta_{t-1}Y_t)-h^*(\eta_{t-1}Y_{t-1})-\left<\nabla h^*(\eta_{t-1}Y_{t-1}) , \eta_{t-1}Y_t-\eta_{t-1}Y_{t-1} \right>\]
(called a Bregman divergence) appear in the first above sum by 
by subtracting
\[ \frac{\left< \eta_{t-1}Y_t-\eta_{t-1}Y_{t-1} , \nabla h^*(\eta_{t-1}Y_{t-1}) \right>}{\eta_{t-1}}=\left< r_t , x_t \right>. \]
Therefore,
\begin{multline*}
  \left< Y_T , x \right> \leqslant \frac{h^*(0)}{\eta_0}+\sum_{t=1}^T\frac{D_{h^*}(\eta_{t-1}Y_t,\eta_{t-1}Y_{t-1})}{\eta_{t-1}}+\sum_{t=1}^T\left< r_t , x_t \right>\\-\frac{\min_{x\in \mathcal{X}}h(x)}{\eta_T}+\frac{\min_{x\in \mathcal{X}}h(x)}{\eta_0}+\frac{\max_{x\in \mathcal{X}}h(x)}{\eta_T}. 
\end{multline*}
Since $h^*(0)=-\min_{x\in \mathcal{X}}h(x)$, we get
\begin{align*}
\operatorname{Reg}_T&=\max_{x\in \mathcal{X}}\left< Y_T , x \right> -\sum_{t=1}^T\left< r_t , x_t \right>\\ 
&\leqslant \frac{\max_{\mathcal{X}}h-\min_{x\in
                                                                                                                             \mathcal{X}}h(x)}{\eta_T}+\sum_{t=1}^T\frac{D_{h^*}(\eta_{t-1}Y_t,\eta_{t-1}Y_{t-1})}{\eta_{t-1}}\\
  &\leqslant \frac{\Delta}{\eta_T}+\sum_{t=1}^T\frac{D_{h^*}(\eta_{t-1}Y_t,\eta_{t-1}Y_{t-1})}{\eta_{t-1}}.
\end{align*}

The strong convexity of the regularizer $h$ let us bound the above
Bregman divergences as follows---see e.g.~\cite[Lemma~13]{shalev2007online}:
\[ D_{h^*}(\eta_{t-1}Y_t,\eta_{t-1}Y_{t-1})\leqslant \frac{1}{2K}\left\| \eta_{t-1}Y_t-\eta_{t-1}Y_{t-1} \right\|_{*}^2=\frac{\eta_{t-1}^2}{2K}\left\| r_t \right\|_{*}^2,\quad t\geqslant 1. \]

Then, set $\eta=\sqrt{\Delta/M^2}$ so that $\eta_t=\eta\, t^{-1/2}$
for $t\geqslant 1$, which is indeed a nonincreasing sequence. The regret bound then becomes
\[ \frac{\Delta\sqrt{T}}{\eta}+\frac{M^2}{2K}\sum_{t=1}^T\eta_{t-1}. \]
We bound the above sum as follows. Since $\eta_0=\eta_1=\eta$,
\begin{align*}
\sum_{t=1}^T\eta_{t-1}&=\eta\left( 2+\sum_{t=2}^{T-1}\frac{1}{\sqrt{t}} \right)\leqslant \eta\left( \int_0^1\frac{1}{\sqrt{s}}\,\mathrm{d}s+\int_1^{T-1}\frac{1}{\sqrt{s}}\,\mathrm{d}s \right)\\  
&=\eta\int_0^{T-1}\frac{1}{\sqrt{s}}\,\mathrm{d}s=2\eta\sqrt{T-1}\leqslant 2\eta\sqrt{T}.
\end{align*}
Injecting the expression of $\eta$ and simplifying gives the result:
\[ \operatorname{Reg}_T\leqslant 2M\sqrt{\frac{T\Delta}{K}}. \]

\input{blackwell}
\input{mixed-actions}

\section{An algorithm for arbitrary norm global cost}
\label{sec:algo-global-cost}
Let \(q'\in (1,2]\). We consider on \(\mathcal{X}=\mathcal{B}\cap \mathcal{C}^\circ\) the
\(\ell_{q'}\) regularizer introduced in Section~\ref{sec:regularizers}:
\[ h_{q'}(x)=
  \begin{cases}
    \frac{1}{2}\left\| x \right\|_{q'}^2&\text{if $x\in \mathcal{X}$} \\
    +\infty&\text{otherwise},
  \end{cases},\quad x\in \mathcal{V}. \]
Let $(\eta_t)_{t\geqslant 1}$ a positive
sequence, and $a$ the oracle from Remark~\ref{rk:oracle}. The
algorithm then writes, for $t\geqslant 1$,
\begin{align*}
\text{compute}\quad x_t&=\nabla h_{q'}^*\left(  \eta_{t-1}\sum_{s=1}^{t-1}r_s\right)\\
\text{compute}\quad a_t&=\argmin_{a\in
                         \Delta_d}\sum_{i=1}^d\max_{}(0,z_{ti}a_i+z'_{ti}),\quad
                         \text{where $(z_t,z_t')=x_t$,}\\
\text{observe}\quad r_t&:=r(a_t,b_t).
\end{align*}

\begin{theorem}[Regret bound for an arbitrary norm cost function]
\label{thm:regret-norm}
Let $q'\in (1,2]$ and $\Delta>0$ such that $\max_{x\in \mathcal{X}}\frac{1}{2}\left\| x \right\|_{q'}^2\leqslant \Delta$. 
Then, the above algorithm with coefficients
\[ \eta_t=d^{1/q'-1}\sqrt{\frac{\Delta(q'-1)}{t}},\quad t\geqslant 1, \]
guarantees, against any sequence $(\ell_t)_{t\geqslant 1}$ in $[0,1]^d$ chosen by the Environment,
\[ \forall T\geqslant 1,\quad \bar{\operatorname{Reg}}_T\leqslant 2\, d^{1-1/q'}\sqrt{\frac{\Delta }{(q'-1) T}}. \]
\end{theorem}
\begin{proof}
We aim at applying Theorem~\ref{thm:approachability-mirror-descent}.
According to Proposition~\ref{prop:lp-regularizer}, because $q'\in (1,2]$, regularizer $h_{q'}$ is $(q'-1)/d^{2(1-1/q')}$ strongly-convex with respect to $\left\|\,\cdot\,\right\|_1$.
Besides, the payoff function $r$ is bounded by $1$ with respect to
$\left\|\,\cdot\,\right\|_\infty$. Indeed, for all $a\in \Delta_d$ and $\ell\in [0,1]^d$,
\[ \left\|   r(a,\ell)\right\|_{\infty}=\left\| (a\odot \ell,\ell) \right\|_{\infty}=\max_{}\left( \left\| a\odot \ell \right\|_{\infty},\left\| \ell \right\|_{\infty} \right)\leqslant 1. \]
And because $0\in \mathcal{X}$, we have the difference between the highest and the lowest values of $h_{q'}$ on its domain bounded from above as
\[ \max_{x\in \mathcal{X}}h_{q'}(x)-\min_{x\in \mathcal{X}}h_{q'}(x)=\max_{x\in \mathcal{X}}\frac{1}{2}\left\| x \right\|_{q'}^2-\min_{x\in \mathcal{X}}\frac{1}{2}\left\| x \right\|_{q'}^2=\max_{x\in \mathcal{X}}\left\| x \right\|_{q'}^2\leqslant \Delta. \]
Therefore, applying Theorem~\ref{thm:approachability-mirror-descent} with $K=(q'-1)/d^{2(1-1/q')}$, $M=1$ and norm $\left\|\,\cdot\,\right\|_1$, together with Proposition~\ref{prop:global-cost-support-function}, gives the result.
\end{proof}

\input{combinatorial}
\input{swap}


%% file: blackwell.tex
\section{Blackwell's algorithm}
\label{sec:repr-blackw-algor}
We recall the definition of Blackwell's algorithm~\citep{blackwell1956analog} and show that it
belongs to the family of FTRL algorithms defined in Section~\ref{sec:defin-analys-algor}.
In the related work by~\cite{shimkin2016online}, it is demonstrated that
Blackwell's algorithm can also be interpreted as a Follow the Leader
algorithm, as well as a FTRL algorithm, in
the context of online convex optimization algorithms converted into
algorithms for the approachability of bounded convex target sets.

We consider \(\mathcal{V}=\mathcal{V}^*=\mathbb{R}^d\) equipped with its
Euclidean structure.
Let \(\mathcal{C}\subset \mathbb{R}^d\) be a closed convex cone which we assume
to be a B-set for the game \((\mathcal{A},\mathcal{B},r)\) and \(a:\mathcal{C}^\circ \to \mathcal{X}\)
a \((\mathcal{A},\mathcal{B},r,\mathcal{C})\)-oracle.
It follows from Definition~\ref{def:B-set} that it is always possible
to choose an oracle \(a\) that satisfies
\begin{equation}
\label{eq:2}
x=\lambda x' \text{ for some $\lambda>0$}\quad \Longrightarrow \quad a(x)=a(x'),\quad x,x'\in \mathcal{C}^\circ.
\end{equation}
We assume in this section that oracle \(a\) satisfies this property.

Blackwell's algorithm \citep{blackwell1954controlled} is defined by
\[ a_t=a\left( \bar{r}_{t-1}-\proj_{\mathcal{C}}\bar{r}_{t-1} \right),\quad t\geqslant 1,  \]
where \(\proj_{\mathcal{C}}\) denotes the Euclidean projection
onto \(\mathcal{C}\). It can be rewritten, using
Proposition~\ref{prop:polar-euclidean}, as
\[ a_t=a\left( \proj_{\mathcal{C}^\circ }\bar{r}_{t-1} \right),\quad t\geqslant 1.  \]

\begin{theorem}
\label{thm:blackwell-mirror-descent}
Let $\mathcal{X}=\mathcal{C}^\circ \cap \mathcal{B}$ where $\mathcal{B}$
denotes the closed Euclidean ball, 
and $h_2$ the Euclidean regularizer on $\mathcal{X}$.
Blackwell's algorithm and the FTRL algorithm associated
with $h_2$ and any
sequence of positive parameters $(\eta_t)_{t\geqslant 1}$ coincide. In other words,
\[ a\left( \bar{r}_{t-1}-\proj_{\mathcal{C}}\bar{r}_{t-1} \right)=a\left( \nabla h_2^*\left( \eta_{t-1}\sum_{s=1}^{t-1}r_s \right) \right),\quad t\geqslant 1.   \]
\end{theorem}
\begin{proof}
Recall that the Euclidean projection $\proj_{\mathcal{E}}w$ of a point $w$ on a closed convex set $\mathcal{E}$ is
the only point in $\mathcal{E}$ satisfying
\begin{equation}
\label{eq:charactzation-projection}
\forall w'\in \mathcal{E},\quad \left< w-\proj_{\mathcal{E}}w , w'-\proj_{\mathcal{E}}w \right> \leqslant 0.
\end{equation}
This characterization will be needed later.

Remember from Proposition~\ref{prop:lp-regularizer} that $\nabla h_2^*=\proj_{\mathcal{C}^\circ \cap \mathcal{B}}$.
Since oracle $a$ satisfies property (\ref{eq:2}), it is enough to prove that for all $u\in \mathbb{R}^d$ and  $\mu>0$, 
\[ \proj_{\mathcal{C}^\circ }u\in \mathbb{R}_+^*\proj_{\mathcal{C}^\circ \cap \mathcal{B}}(\mu u ). \]
Besides, $\mathcal{C}^\circ$ being a closed convex cone, $\proj_{\mathcal{C}^\circ }(\mu u)=\mu \proj_{\mathcal{C}^\circ}u$.
It is therefore equivalent to prove that for all $w\in \mathbb{R}^d$,
\begin{equation}
\label{eq:second-step}
\proj_{\mathcal{C}^\circ}w\in \mathbb{R}_+^* \proj_{\mathcal{C}^\circ\cap  \mathcal{B}}w.
\end{equation}
Let $w\in \mathbb{R}^d$.
If $\left\| \proj_{\mathcal{C}^\circ}w \right\|_2\leqslant 1$, then
obviously $\proj_{\mathcal{C}^\circ}w=\proj_{\mathcal{C}^\circ\cap
  \mathcal{B}}w$ as shown in Figure~\ref{fig:blackwell-mirror-descent-easy} and (\ref{eq:second-step}) is true.
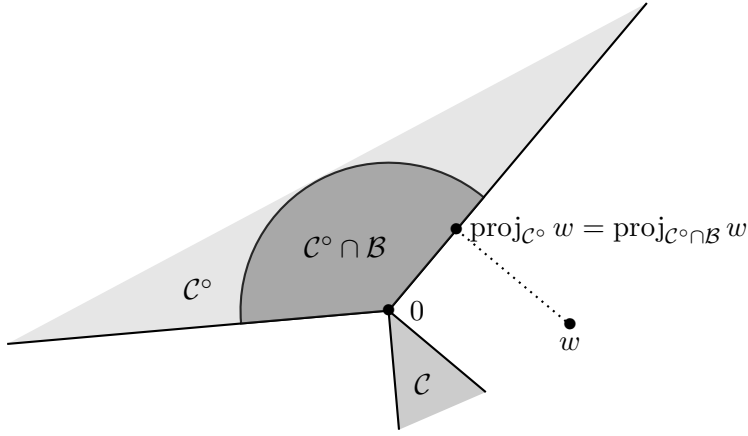
\begin{figure}
  \centering
\begin{tikzpicture}[y=0.80pt, x=0.80pt, yscale=-1.000000, xscale=1.000000, inner sep=0pt, outer sep=0pt,scale=2,rotate=130]
  \path[draw=black,fill=ccccccc,line join=miter,line cap=butt,even odd rule,line
    width=0.800pt] (0.0000,1022.3622) -- (0.0000,1052.3622) --
    (20.0000,1032.3622);
  \path[draw=black,fill=ce6e6e6,line join=miter,line cap=butt,even odd rule,line
    width=0.800pt] (-95.0000,1052.3622) -- (0.0000,1052.3622) --
    (64.0000,1116.3622);
  \path[draw=black,fill=c999999,opacity=0.800,line width=.8pt]
    (24.7487,1077.1109)arc(45.000:112.500:35.000)arc(112.500:180.000:35.000) --
    (0.0000,1052.3622) -- cycle;
  \path[draw=black,line join=miter,line cap=butt,even odd rule,line width=0.800pt,dotted]
    (-25.0000,1052.3622) node {$\bullet$} node[right=5pt]
    {$\proj_{\mathcal{C}^\circ}w=\proj_{\mathcal{C}^\circ\cap
        \mathcal{B}}w$} -- (-25.0000,1017.3622) node {$\bullet$}
    node[below=5pt] {$w$};
    \draw (-5,1070) node {$\mathcal{C}^\circ\cap \mathcal{B}$};
    \draw (25,1090) node {$\mathcal{C}^\circ$};
    \draw (8,1035) node {$\mathcal{C}$};
    \draw (0.0000,1052.3622) node {$\bullet$} node[right=8pt] {$0$};

\end{tikzpicture}
\caption{In the case where $\left\| \proj_{\mathcal{C}^\circ }w \right\|_2\leqslant  1$, we have $\proj_{\mathcal{C}^\circ}w=\proj_{\mathcal{C}^\circ \cap \mathcal{B}}w$.}
\label{fig:blackwell-mirror-descent-easy}
\end{figure}
We now assume that $\left\| \proj_{\mathcal{C}^\circ}w \right\|_2>
1$. 
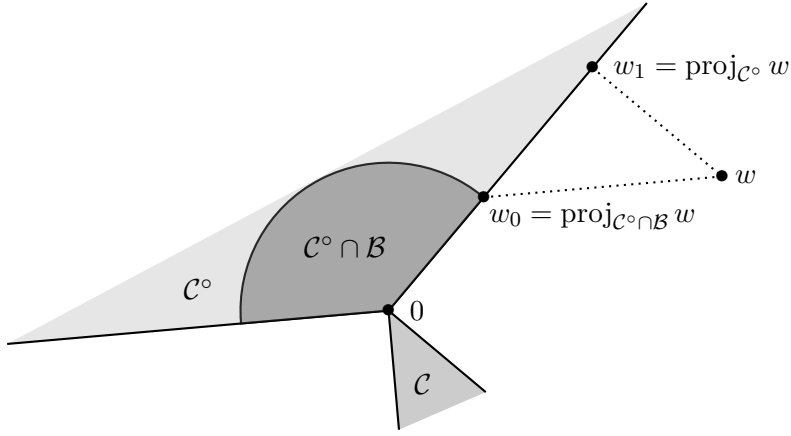
\begin{figure}
  \centering
\begin{tikzpicture}[y=0.80pt, x=0.80pt, yscale=-1.000000, xscale=1.000000, inner sep=0pt, outer sep=0pt,rotate=130,scale=2]
  \path[draw=black,fill=ccccccc,line join=miter,line cap=butt,even odd rule,line
    width=0.800pt] (0.0000,1022.3622) -- (0.0000,1052.3622) --
    (20.0000,1032.3622);
  \path[draw=black,fill=ce6e6e6,line join=miter,line cap=butt,even odd rule,line
    width=0.800pt] (-95.0000,1052.3622) -- (0.0000,1052.3622) --
    (64.0000,1116.3622);
  \path[draw=black,fill=c999999,opacity=0.800,line width=.8pt]
    (24.7487,1077.1109)arc(45.000:112.500:35.000)arc(112.500:180.000:35.000) --
    (0.0000,1052.3622) -- cycle;

  \draw[draw=black,line join=miter,line cap=butt,even odd rule,line width=0.800pt,dotted]
    (-75.0000,1052.3622) node {$\bullet$} node[right=8pt]
    {$w_1=\proj_{\mathcal{C}^\circ }w$} -- (-75.0000,1012.3622) node
    {$\bullet$} node[right=5pt] {$w$};
  \draw[draw=black,line join=miter,line cap=butt,even odd rule,line width=0.800pt,dotted]
    (-35.0000,1052.3622) node {$\bullet$} node[below right=2pt] {$w_0=\proj_{\mathcal{C}^\circ \cap \mathcal{B}}w$} -- (-75.0000,1012.3622);

    \draw (-5,1070) node {$\mathcal{C}^\circ\cap \mathcal{B}$};
    \draw (25,1090) node {$\mathcal{C}^\circ$};
    \draw (8,1035) node {$\mathcal{C}$};
    \draw (0.0000,1052.3622) node {$\bullet$} node[right=8pt] {$0$};
\end{tikzpicture}
    \caption{In the case where $\left\| \proj_{\mathcal{C}^\circ }w
      \right\|_2> 1$, we have $w_0=\proj_{\mathcal{C}^\circ \cap \mathcal{B}}w$.}
  \label{fig:blackwell-mirror-descent}
\end{figure}
We define
\[ w_0:=\frac{\proj_{\mathcal{C}^\circ}w}{\left\|  \proj_{\mathcal{C}^\circ}w \right\|_2}. \]
Using characterization (\ref{eq:charactzation-projection}), we aim at
proving that $w_0=\proj_{\mathcal{C}^\circ\cap  \mathcal{B}}w$ (see Figure~\ref{fig:blackwell-mirror-descent}), which would prove (\ref{eq:second-step}). 
First, $w_0$ belongs to $\mathcal{C}^{\circ }\cap \mathcal{B}$ by definition.
Let $w'\in \mathcal{C}^\circ \cap \mathcal{B}$. For short, denote $w_1=\proj_{\mathcal{C}^\circ}w$.
\begin{align*}
\left< w-w_0 , w'-w_0 \right> &=\left< w-w_1+w_1-w_0 , w'-w_0 \right>\\ 
&=\left< w-w_1 , w'-w_0 \right> +\left< w_1-w_0 , w'-w_0 \right>\\ 
&=\frac{1}{\left\| w_1 \right\| }\left< w-w_1 , \left\| w_1 \right\|_{}w'-w_1 \right>+\left< w_1-w_0 , w'-w_0 \right>. 
\end{align*}
The first dot product above is nonpositive by characterization of $w_1=\proj_{\mathcal{C}^\circ}w$, because $\left\| w_1 \right\|_{}w'\in \mathcal{C}^\circ$.
Let us prove that the second dot product is also nonpositive. For all $w''\in \mathcal{C}^\circ \cap \mathcal{B}$,
\[ \left\| w_1-w'' \right\|_{}\geqslant \left| \left\| w_1 \right\| -\left\| w'' \right\|  \right| \geqslant \left\| w_1 \right\| -1=\left\| w_1-w_0 \right\|, \]
which means that $w_0=\proj_{\mathcal{C}^\circ\cap  \mathcal{B}}w_1$. Thus, $\left< w_1-w_0 , w'-w_0 \right> \leqslant 0$.
Therefore,
\[ \left< w-w_0 , w'-w_0 \right>\leqslant 0 \] and (\ref{eq:second-step}) is proved.
\end{proof}

We can now recover via
Theorem~\ref{thm:approachability-mirror-descent} the classic guarantee
for Blackwell's algorithm in the case where the vector payoffs are
bounded with respect to the Euclidean norm.

\begin{theorem}
\label{thm:blackwell}
Let $M>0$. Assume that $\left\| r(a,b) \right\|_2\leqslant M$ for all $a\in \mathcal{A}$ and $b\in \mathcal{B}$.
Then, against any sequence of actions $(b_t)_{t\geqslant 1}$ chosen by the Environment, Blackwell's algorithm guarantees:
\[ \forall T\geqslant 1,\quad d_2\left( \bar{r}_T,\ \mathcal{C} \right)\leqslant \frac{2\sqrt{2}M}{\sqrt{T}},  \]
where $d_2$ denotes the Euclidean distance.
\end{theorem}
\begin{proof}
With notation from Theorem~\ref{thm:blackwell-mirror-descent}, we have $\max_{x\in \mathcal{X}}h_2(x)-\min_{x\in \mathcal{X}}h_2(x)=1/2$, and $h_2$ is $1$-strongly convex with respect to $\left\|\,\cdot\,\right\|_2$ by Proposition~\ref{prop:lp-regularizer}.
According to Theorem~\ref{thm:blackwell-mirror-descent}, Blackwell's algorithm corresponds to the FTRL algorithm associated with $h_2$ and any sequence of parameters $(\eta_t)_{t\geqslant 1}$.
We can therefore apply
Theorem~\ref{thm:approachability-mirror-descent} with $\Delta=1/2$ and
$K=1$, together with Proposition~\ref{prop:distance-support} and the result follows.
\end{proof}

%% file: mixed-actions.tex
\section{Model with mixed actions}
\label{sec:model-with-mixed}
We here present a variant of the model from Section~\ref{sec:model}, in which the decision
maker has a finite set of \emph{pure actions} \(\mathcal{I}=\left\{ 1,\dots,d
\right\}\) from which he is allowed to choose at random. We define the
corresponding FTRL algorithms, and state
guarantees in expectation, with high probability, and almost-surely.
Let the simplex \(\Delta_d=\left\{ x\in \mathbb{R}_+^d,\
  \sum_{i=1}^dx_i=1  \right\}\) be the set of \emph{mixed actions}
(which we identify to the set
of probability distributions over $\mathcal{I}$), \(\mathcal{B}\) a set of
actions for the Environment, and \(r\colon \mathcal{I}\times \mathcal{B}\to \mathbb{R}^d\) a payoff function.
We linearly extend the payoff function \(r\) in its first variable:
\[ r(a,b):=\mathbb{E}_{i\sim a}\left[ r(a,b) \right]=\sum_{i=1}^{d}a_ir(i,b),\quad a\in \Delta_d,\ b\in \mathcal{B}.  \]
The game is played as follows. At time \(t\geqslant 1\), 
\begin{itemize}
\item the Decision Maker chooses  mixed action $a_t\in \Delta_d$;
\item the Environment chooses action $b_t\in \mathcal{B}$;
\item the Decision Maker draws pure action $i_t\sim a_t$;
\item the Decision Maker observes vector payoff $r_t:=r(i_t,b_t)$.
\end{itemize}
Denote \((\mathcal{F}_t)_{t\geqslant 1}\) the filtration where
\(\mathcal{F}_t\) is generated by
\[ (a_1,b_1,i_1,\dots,a_{t-1},b_{t-1},i_{t-1},a_t,b_t). \]
An algorithm for the Decision Maker is a sequence of maps
\(\sigma=(\sigma_t)_{t\geqslant 1}\) where \(\sigma_t:(\Delta_d\times
\mathcal{I}\times \mathcal{V}^*)^{t-1}\to \Delta_d\) so that action $a_t$ is given by
\[ a_t=\sigma_t(a_1,i_1,r_1,\dots,a_{t-1},i_{t-1},r_{t-1}),\quad t\geqslant 1. \]
Regarding the Environment, we assume that its choice of action
\(b_t\) does not depend on \(i_t\), so that \(\mathbb{E}\left[ r(i_t,b_t)\,\middle|\,\mathcal{F}_t \right]=\mathbb{E}_{i\sim a_t}\left[ r(i,b_t) \right]=r(a_t,b_t)\).
In this model, Blackwell's condition writes as follows.
\begin{definition}[Blackwell's condition for games with mixed actions]
A closed convex cone $\mathcal{C}$ of the payoff space $\mathcal{V}^*$
is a \emph{B-set for the game with mixed actions $(\mathcal{I},\mathcal{B},r)$} if
\[ \forall x\in \mathcal{C}^\circ ,\ \exists \, a(x)\in \Delta_d,\ \forall b\in \mathcal{B},\quad \left< r(a(x),b) , x \right> \leqslant 0. \]
Such an application $a:\mathcal{C}^\circ \to \Delta_d$ is called a \emph{$(\mathcal{I},\mathcal{B},r,\mathcal{C})$-oracle}.
\end{definition}

We can now define the FTRL algorithms similarly as in
Section~\ref{sec:defin-analys-algor}. Let \(\mathcal{C}\) be a closed
convex cone of the payoff space \(\mathcal{V}^*\) which is assumed to be
a B-set for the game with mixed actions \((\mathcal{I},\mathcal{B},r)\), \(a:\mathcal{C}^\circ \to \Delta_d\) a
\((\mathcal{I},\mathcal{B},r,\mathcal{C})\)-oracle, \(\mathcal{X}\) a generator of
\(\mathcal{C}^\circ\), \(h\) a regularizer on \(\mathcal{X}\), and 
\((\eta_t)_{t\geqslant 1}\) a positive sequence. Then, the
corresponding algorithm writes, for \(t\geqslant 1\),
\begin{align*}
\text{compute}\quad x_t&=\nabla h^*\left( \eta_{t-1}\sum_{s=1}^{t-1}r_s \right)\\
\text{compute}\quad a_t&=a(x_t)\\
\text{draw}\quad i_t&\sim a_t\\
\text{observe}\quad r_t&=r(i_t,b_t).
\end{align*}
\begin{theorem}
\label{thm:mirror-descent-approachability-mixed-actions}
Let $\Delta,M,K>0$, $\left\|\,\cdot\,\right\|$ be a norm on
$\mathcal{V}$, and $\left\|\,\cdot\,\right\|_*$ its dual norm on $\mathcal{V}^*$. We assume:
\begin{enumerate}[(i)]
\item $\max_{x\in \mathcal{X}}h(x)-\min_{x\in \mathcal{X}}h(x)\leqslant \Delta$,
\item $h$ is $K$-strongly convex with respect to $\left\|\,\cdot\,\right\|$,
\item $\left\| r(a,b) \right\|_* \leqslant M$ for all $a\in \Delta_d$ and $b\in \mathcal{B}$.
\end{enumerate}
Then the above algorithm guarantees, with the choice $\eta_t=\sqrt{\Delta
  K/M^2t}$ (for $t\geqslant 1$), against any sequence of actions $(b_t)_{t\geqslant 1}$
chosen by the Environment:
\[ \forall T\geqslant 1,\quad \mathbb{E}\left[ I_{\mathcal{X}}^*\left( \bar{r}_T \right)  \right]\leqslant 2M\sqrt{\frac{\Delta}{KT}}.  \]
Let $\delta\in (0,1)$. For all $T\geqslant 1$, we have with probability higher than $1-\delta$,
\[ I_{\mathcal{X}}^*\left( \bar{r}_T \right)\leqslant \frac{M}{\sqrt{T}}\left( 2\sqrt{\frac{\Delta}{K}}+\left\| \mathcal{X} \right\|_{}\sqrt{2\log (1/\delta)} \right).  \]
Almost-surely,
\[ \limsup_{T\to +\infty}I_{\mathcal{X}}^*\left( \bar{r}_T \right)\leqslant 0. \]
\end{theorem}
\begin{proof}
Like in the proof of Theorem~\ref{thm:approachability-mirror-descent}, Lemma~\ref{lm:regret-bound} gives:
\begin{equation}
\label{eq:1}
I_{\mathcal{X}}^*(\bar{r}_T)\leqslant \frac{1}{T}\left( \sum_{t=1}^T\left< r_t
    , x_t \right> +2M\sqrt{\frac{\Delta T}{K}} \right). 
\end{equation}
Consider $X_t=\left< r_t , x_t \right>$. Then, $(X_t)_{t\geqslant 1}$ is a sequence of super-martingale differences with respect to filtration $(\mathcal{F}_t)_{t\geqslant 0}$:
\[ \mathbb{E}\left[ \left< r_t , x_t \right> \,\middle|\,\mathcal{F}_t \right]=\mathbb{E}\left[ \left< r(i_t,b_t) , x_t \right> \,\middle|\,\mathcal{F}_t \right]=\left< \mathbb{E}\left[ r(i_t,b_t)\,\middle|\,\mathcal{F}_t \right]  , x_t \right> =\left< r(a_t,b_t) , x_t \right> \leqslant 0,    \]
because $x$ is a $(\mathcal{I},\mathcal{B},r,\mathcal{C})$-oracle. Therefore,
\[ \mathbb{E}\left[ \sum_{t=1}^T\left< r_t , x_t \right>  \right]=\mathbb{E}\left[ \sum_{t=1}^T\mathbb{E}\left[ \left< r_t , x_t \right> \,\middle|\,\mathcal{F}_{t} \right]  \right]\leqslant 0.   \]
Injecting this in Equation (\ref{eq:1}) gives the guarantee in expectation:
\[ \mathbb{E}\left[ I_{\mathcal{X}}^*(\bar{r}_T)\right]\leqslant 2M\sqrt{\frac{\Delta}{KT}}.  \]
We now turn to the high probability bound. Let $\delta\in (0,1)$. From Equation (\ref{eq:1}), we deduce that
\[ I_{\mathcal{X}}^*(\bar{r}_T)\leqslant 2M\sqrt{\frac{\Delta}{KT}}+\frac{1}{T}\sum_{t=1}^TX_t. \]
Since we have $\left| X_t \right| =\left| \left< r(i_t,b_t) , x_t \right>  \right| \leqslant \left\| r(i_t,b_t) \right\|_{*}\left\| x_t \right\|_{}\leqslant M\left\| \mathcal{X} \right\|_{}$ for all $t\geqslant 1$,
the Azuma--Hoeffding inequality assures that with probability higher than $1-\delta$,
\[ \frac{1}{T}\sum_{t=1}^TX_t\leqslant M\left\| \mathcal{X} \right\|_{} \sqrt{\frac{2\log (1/\delta)}{T}} \]
and thus
\[ I_{\mathcal{X}}^*(\bar{r}_T)\leqslant \frac{M}{\sqrt{T}}\left( 2\sqrt{\frac{\Delta}{K}}+\left\| \mathcal{X} \right\|_{}\sqrt{2\log (1/\delta)} \right).   \]
The almost-sure guarantee follows from a standard Borel--Cantelli argument.
\end{proof}

%% file: combinatorial.tex
\section{Online combinatorial optimization}
\label{sec:online-comb-optim}
We illustrate the flexibility of our general framework by giving an
alternative construction of an optimal algorithm in the the online
combinatorial optimization problem with full information feedback. It
is a regret minimization problem in which the actions and the payoffs
have a particular structure.  Numerous papers were written on the
topic,
including~\cite{gentile1998linear,kivinen2001relative,grove2001general,takimoto2003path,kalai2005efficient,warmuth2008randomized,helmbold2009learning,hazan2010learning}.
A minimax optimal algorithm was given in~\cite{koolen2010hedging}.  We
give below an alternative construction of such an algorithm.

Let \(d,m\geqslant 1\) be integers. Let \(\mathcal{I}=\left\{ 1,\dots,d
  \right\}\) be a finite set. The set of pure actions 
  of the Decision Maker is a set \(P\) which contains subsets of
  \(\mathcal{I}\) of cardinality \(m\). Denote \(\Delta(P)\) the unit
  simplex in $\mathbb{R}^P$ and let it be the set of
  mixed actions by identifying it to the set of probability
  distributions over $P$.
The game is played as follows. At time \(t\geqslant 1\), the Decision Maker
\begin{itemize}
\item chooses mixed action $a_t\in \Delta(P)$;
\item draws pure action $p_t\sim  a_t$;
\item observes payoff vector $v_t\in \mathbb{R}^d$;
\item gets payoff $\sum_{i\in p_t}^{}v_{ti}$.
\end{itemize}
We assume that the choice by the Environment of payoff vector \(v_t\in
\mathbb{R}^d\) does not depend on pure action \(p_t\).
The quantity to minimize is the following regret:
  \[ \operatorname{Reg}_T=\max_{p\in P}\sum_{t=1}^T \sum_{i\in p}^{}v_{ti}-\sum_{t=1}^T \sum_{i\in p_t}^{}v_{ti}. \]

This problem can be seen as a basic regret minimization problem with pure action set
\(P\), and payoff vectors \((\sum_{i\in p}^{}v_i)_{p\in P}\) which belong
to \([-m,m]^P\) as soon as we assume \(v\in [-1,1]^d\). The classical Exponential
Weights Algorithm~\citep{cesa1997analysis} would then guarantee a regret
bound of order \(m\sqrt{T\log\left| P \right|}\).  However, our goal is
to take advantage of the structure of the problem and to construct a
algorithm which guarantees a significantly tighter regret bound, of
order \(m\sqrt{T\log (d/m)}\), which is known to be minimax
optimal~\citep{koolen2010hedging}. To do so, we reduce this problem to
a well-chosen approachability game (with mixed actions, as in Section~\ref{sec:model-with-mixed}), which we now present.

Let \(A\) be the \(d\times \left| P \right|\) matrix defined by
\(A=(\mathbbm{1}_{\left\{ i\in p \right\} })_{\substack{i\in
\mathcal{I}\\p\in P}}\), and for each \(p\in P\), let
\(e_p=(\mathbbm{1}_{\left\{ i\in p\right\}})_{i\in \mathcal{I}}\in
\mathbb{R}^d\).  Let \(P\) (resp. \(\Delta(P)\)) be the set of pure
(resp. mixed) actions for the Decision Maker,
\(\mathcal{B}=[-1,1]^d\) the set of actions for the Environment,
and consider the following payoff function: 
\[ r(p,v)=v-\frac{\left< v , e_p \right> }{m}\mathbf{1}\in \mathbb{R}^d,\quad p\in P,\ v\in [-1,1]^d, \] 
 where \(\mathbf{1}=(1,\dots,1)\in
\mathbb{R}^d\). The payoff space is therefore
\(\mathcal{V}^*=\mathbb{R}^d\). 
The linear extension of the payoff function in its first variable writes
\[ r(a,v)= v-\frac{\left< v
, Aa \right> }{m}\mathbf{1},\quad a\in \Delta(P),\ v\in [-1,1]^d.\]
We now choose the generator: let \(\mathcal{X}=A(\Delta(P))\) be the image of
the simplex \(\Delta(P)\) via \(A\) seen as a linear map from
\(\mathbb{R}^P\) to \(\mathbb{R}^d\). Its properties are gathered in the
following proposition. In particular, property (\ref{item:15})
demonstrates that this choice of \(\mathcal{X}\) makes
\(I_{\mathcal{X}}^*(\bar{r}_T)\) equal to the above regret.
\begin{proposition}
\label{prop:combinatorial-Z}
\begin{enumerate}[(i)]
\item\label{item:12} $\mathcal{X}$ is the convex hull of the points $e_p$ ($p\in P$).
\item\label{item:13} $\mathcal{X}\subset m\Delta_d$.
\item\label{item:14} $\left\| \mathcal{X} \right\|_1= m$.
\item\label{item:22} $\mathcal{X}$ is a generator of $\mathcal{X}^{\circ \circ} =A(\Delta(P))^{\circ \circ}$.
\item\label{item:15} Let $(p_t)_{t\geqslant 1}$ be a sequence of pure actions
  chosen by the Decision Maker and  $(v_t)_{t\geqslant 1}$ a sequence of
  actions chosen by the Environment,
and denote $r_t=r(p_t,v_t)$ for all $t\geqslant 1$ the corresponding
payoffs. Then, for all $T\geqslant 1$,
\[ I_{\mathcal{X}}^*\left( \bar{r}_T \right)=\frac{1}{T}\operatorname{Reg}_T=\frac{1}{T}\left( \max_{p\in P}\sum_{t=1}^T \sum_{i\in p}^{}v_{ti}-\sum_{t=1}^T \sum_{i\in p_t}^{}v_{ti}\right). \]
\end{enumerate}
\end{proposition}
\begin{proof}
By definition, $\mathcal{X}$ is the image of simplex $\Delta(P)$ via linear map $A$.
It is therefore the convex hull of the image by $A$ of the extreme points of $\Delta(P)$. And 
for $p_0\in P$, $A(\mathbbm{1}_{\left\{ p=p_0\right\} })_{p\in P}=e_p$. Hence (\ref{item:12}).
Each point $e_p$ clearly belongs to $m\Delta_d$, and (\ref{item:13}) is true by convexity of $m\Delta_d$.
For each element $x\in m\Delta_d$, we have $\left\| x \right\|_1=m$, which implies (\ref{item:14}). 
$\mathcal{X}$ is a nonempty convex compact set thanks to (\ref{item:12}); Proposition~\ref{prop:generator-examples} gives (\ref{item:22}).
As for
the relation (\ref{item:15}), we denote $A^*$ the transpose of $A$ and write
\begin{align*}
\max_{p\in P}\sum_{t=1}^T \sum_{i\in p}^{}v_{ti}-\sum_{t=1}^T \sum_{i\in p_t}^{}v_{ti} &= \max_{p\in P}\sum_{t=1}^T\left(  (A^*v_t)_p- (A^*v_t)_{p_t}\right)\\
&=\max_{a\in \Delta(P)}\sum_{t=1}^T \left(  \left< A^*v_t , a \right> -\left< A^*v_t , \left( \mathbbm{1}_{\left\{ p=p_t \right\} } \right)_{p\in P}  \right>\right)\\ 
&=\max_{a\in \Delta(P)}\sum_{t=1}^T \left(  \left< v_t , Aa \right> -\left< v_t , A\left(  \mathbbm{1}_{\left\{ p=p_t \right\} }\right)_{p\in P} \right>\right)\\
&=\max_{x\in A(\Delta(P))}\sum_{t=1}^T \left(  \left< v_t , x \right> -\left< v_t , e_{p_t} \right> \right)\\
&=\max_{x\in \mathcal{X}}\sum_{t=1}^T \left< v_t-\frac{\left< v_t , e_{p_t} \right> }{m}\mathbf{1} , x \right>\\ 
&=\max_{x\in \mathcal{X}}\sum_{t=1}^T \left< r(p_t,v_t) , x \right>\\ 
&=T\cdot I_{\mathcal{X}}^*(\bar{r}_T),
\end{align*}
where in the fifth line, we used the fact that for all $x\in
\mathcal{X}$, $\left< \mathbf{1} , x \right> =m$, which is a consequence of (\ref{item:13}).
\end{proof}

\begin{proposition}
$A(\Delta(P))^\circ$ is a B-set for the game with mixed actions $(P,[-1,1]^d,r)$.
\end{proposition}
\begin{proof}
Since $\mathcal{X}$ is a generator of $A(\Delta(P))^{\circ \circ}$, one can check that the condition
that defines a B-set only needs to be verified for $x\in \mathcal{X}$. Let $x\in \mathcal{X}$. By definition
of $\mathcal{X}$, there exists $a\in \Delta(P)$ such that $x=Aa$. Then for $v\in [-1,1]^d$,
\[ \left< r(a,v) , x \right> =\left< v-\frac{\left< v , Aa \right> }{m}\mathbf{1} , Aa \right>=\left< v , Aa \right> -\left< v , Aa \right> =0,  \]
which proves the result.
\end{proof}

As a consequence of Proposition~\ref{prop:combinatorial-Z}, a point \(x\in
 \mathcal{X}\) only has nonnegative components. We can therefore define 
  \[ h(x)= \begin{cases}
   \displaystyle \sum_{i=1}^d\frac{x_i}{m}\log \frac{x_i}{m}&\text{for $x\in
   \mathcal{X}$}\\
   +\infty&\text{otherwise}.
   \end{cases} \]

\begin{proposition}
\label{prop:combinatorial-regularizer}
\begin{enumerate}[(i)]
\item\label{item:23} $h$ is a regularizer on $\mathcal{X}$;
\item\label{item:24} $\max_{x\in \mathcal{X}}h-\min_{x\in \mathcal{X}}h(x)\leqslant \log (d/m)$;
\item\label{item:25} $h$ is $1/m^2$-strongly convex with respect to $\left\|\,\cdot\,\right\|_1$.
\end{enumerate}
\end{proposition}
\begin{proof}
For $x\in \mathcal{X}\subset m\Delta_d$, we can write $h(x)=h_{\mathrm{ent}}(x/m)<+\infty$.
The $1$-strong convexity of $h_{\mathrm{ent}}$ with respect to $\left\|\,\cdot\,\right\|_1$ implies the $1/m^2$-strong convexity of
$h$ with respect to $\left\|\,\cdot\,\right\|_1$ and (\ref{item:25}) is proved. In particular, 
 $h$ is strictly convex. Besides, the domain of $h$ is $\mathcal{X}$ by definition and (\ref{item:23}) is proved.
As for (\ref{item:24}), $h$ being convex, its maximum is attained at one of the extreme points $e_p$ ($p\in P$) of $\mathcal{X}$:
\[ \max_{x\in \mathcal{X}}h(x)=\max_{p\in P}h(e_p)=\max_{p\in P}\sum_{i\in p}^{}\frac{1}{m}\log \frac{1}{m}=-\log m. \]
As for the minimum,
\[ \min_{x\in \mathcal{X}}h(x)\geqslant \min_{x\in m\Delta_d}\sum_{i=1}^d\frac{x_i}{m}\log \frac{x_i}{m}=\min_{x\in \Delta_d}\sum_{i=1}^dx_i\log x_i=-\log d. \]
Therefore, $\max_{x\in \mathcal{X}}h-\min_{x\in \mathcal{X}}h(x)\leqslant -\log m+\log d=\log (d/m)$.
\end{proof}

We can now consider the FTRL algorithm associated with
regularizer \(h\), a \((P,[-1,1]^d,r,A(\Delta(P))^\circ)\)-oracle \(a\), and a positive sequence of
parameters \((\eta_t)_{t\geqslant 1}\), for \(t\geqslant 1\),
\begin{align*}
\text{compute}\quad x_t&=\argmax_{x\in \mathcal{X}}\left\{ \left< \eta_{t-1}\sum_{s=1}^{t-1}r_s , x \right>  -h(x)\right\}\\
\text{choose}\quad a_t&=a(x_t)\\
\text{draw}\quad p_t&\sim a_t\\
\text{observe}\quad r_t&=r(p_t,v_t)=v_t-\frac{\left< v_t , Ae_{p_t} \right>}{m}\mathbf{1}.
\end{align*}
\begin{theorem}
\label{thm:combinatorial}
Against any sequence $(v_t)_{t\geqslant 1}$ in $[-1,1]^d$ chosen by the Environment, the above algorithm with parameters $\eta_t=\sqrt{\log (d/m)/4m^2t}$ (for $t\geqslant 1$)
guarantees
\[ \mathbb{E}\left[ \operatorname{Reg}_T \right]\leqslant 4m\sqrt{T\log (d/m)}. \]
For $\delta\in (0,1)$, we have with probability higher than $1-\delta$,
\[ \operatorname{Reg}_T\leqslant 2m\sqrt{T}\left( 2\sqrt{\log (d/m)}+\sqrt{2\log (1/\delta)} \right).   \]
Almost-surely,
\[ \limsup_{T\to +\infty}\frac{1}{T}\operatorname{Reg}_T\leqslant 0. \]
\end{theorem}
\begin{proof}
For all $v\in [-1,1]^d$ and $p\in P$,
\[ \left\| r(p,v) \right\|_{\infty}=\left\| v-\frac{\left< v , Ae_p \right> }{m}\mathbf{1} \right\|_{\infty}\leqslant \left\| v \right\|_{\infty}+\frac{\left\| \mathbf{1} \right\|_{\infty}}{m}\sum_{i\in p}^{}\left| v_i \right| \leqslant 2. \]
The result then follows from Theorem~\ref{thm:mirror-descent-approachability-mixed-actions} applied with $M=2$, $K=1/m^2$, the properties of the regularizer $h$ given by Proposition~\ref{prop:combinatorial-regularizer}, 
and the relation (\ref{item:15}) from Proposition~\ref{prop:combinatorial-Z}.
\end{proof}

%% file: swap.tex
\section{Internal and swap regret}
\label{sec:internal-swap-regret}
We further illustrate the generality of our framework by recovering
the best known algorithms for internet and swap regret minimization.
The notion of \emph{internal regret} was introduced
by~\cite{foster1997calibrated}. It is an alternative quantity to the
usual regret. \cite{foster1997calibrated} first established the
existence of algorithms which guarantees that the average internal
regret is asymptotically nonpositive (see
also~\cite{fudenberg1995consistency,fudenberg1999conditional,hart2000simple,hart2001general,stoltz2005internal}). \cite{blum2005external}
introduced the swap regret, which generalizes
both the internal and the basic regret. The optimal bound on the swap
regret is known since~\cite{blum2005external,stoltz2005internal}. Later, \cite{perchet2015exponential}
proposed an approachability-based optimal algorithm.
We present below the construction of an algorithm similar to \cite{stoltz2005internal,perchet2015exponential} using the tools
introduced in Sections~\ref{sec:appr-conv-cones} and \ref{sec:class-foll-regul}. The internal regret is mentioned
at the end of the section as a special case.

The set of pure actions of the Decision Maker is \(\mathcal{I}=\left\{
1,\dots,d \right\}\). At time \(t\geqslant 1\), the Decision Maker
\begin{itemize}
\item chooses mixed action $a_t\in \Delta_d$;
\item draws pure action $i_t\sim a_t$;
\item observes payoff vector $v_t\in \mathbb{R}^d$.
\end{itemize}
Let \(\Phi\) be a nonempty subset of \(\mathcal{I}^{\mathcal{I}}\). The
quantity to minimize is the \(\Phi\)-regret defined by
  \[ \operatorname{Reg}_T^{\Phi}=\max_{\varphi \in \Phi}\sum_{t=1}^Tv_{t\varphi (i_t)}-\sum_{t=1}^{T}v_{ti_t}, \]
and can be interpreted as follows.
For a given map \(\varphi \in \Phi\), \(\sum_{t=1}^{T}v_{t\varphi
(i_t)}\) is the cumulative payoff that the Decision Maker would have
obtained if he had played pure action \(\varphi (i)\) each time he has 
actually played \(i\) (for all \(i\in \mathcal{I}\)). The \(\Phi\)-regret
therefore compares the actual cumulative payoff of the Decision Maker
with the best such quantity (for \(\varphi \in \Phi\)) in hindsight.
The goal is to construct an algorithm which guarantees on the
\(\Phi\)-regret a bound of order \(\sqrt{T\log \left| \Phi \right|}\). To
do so, we reduce this problem to a well-chosen approachability game
(with mixed actions as in Section~\ref{sec:model-with-mixed}), which we now present.

Let \(\mathcal{I}\) (resp. \(\Delta_d\)) be the set of pure (resp. mixed)
actions for the Decision Maker and \([-1,1]^d\) the set of actions for
the Environment.
Let the payoff space be \(\mathcal{V}^*=\mathbb{R}^{\Phi}\) and the
target set be \(\mathbb{R}_-^{\Phi}\). We choose the
following payoff function:
\[ r(i,v)=\left(v_{\varphi (i)}-v_i\right)_{\varphi \in \Phi} \in \mathbb{R}^\Phi,\quad i\in \mathcal{I},\ v\in [-1,1]^d. \]
The linear extension of
the payoff function in its first variable is
\[ r(a,v)=\left(\sum_{i\in \mathcal{I}}a_i(v_{\varphi (i)}-v_i) \right)_{\varphi \in \Phi},\quad a\in \Delta_d,\ v\in \mathbb{R}^d. \]
\begin{proposition}
  \label{prop:swap-B-set}
$\mathbb{R}_-^{\Phi}$ is a B-set for the game with mixed actions $(\mathcal{I},[-1,1]^d,r)$.
\end{proposition}
\begin{proof}
Let $x=(x_{\varphi })_{\varphi \in \Phi}\in
(\mathbb{R}_-^{\Phi})^\circ =\mathbb{R}_+^{\Phi}$.
Let us prove that there exists $a\in \Delta(\mathcal{I})$ such that
for all $v\in [-1,1]^d$, $\left< r(a,v) , x \right> \leqslant 0$. First, the property is
trivially true if $x=0$. We assume from now on that $x\neq 0$.

Denote
\[ \tilde{x}_{ij}=\sum_{\substack{\varphi \in \Phi\\\varphi (i)=j}}x_{\varphi },\quad i,j\in \mathcal{I} \]
and let us first prove that there exists $a\in \Delta(\mathcal{I})$
such that:
\begin{equation}
\label{eq:invariant}
\sum_{i\in \mathcal{I}}a_i\tilde{x}_{ij}=a_j\sum_{i\in \mathcal{I}}\tilde{x}_{ji},\quad j\in \mathcal{I}.
\end{equation}
Notice that for all $i\in \mathcal{I}$ we have
\[ \sum_{j\in \mathcal{I}}\tilde{x}_{ij}=\sum_{j\in
    \mathcal{I}} \sum_{\substack{\varphi\in \Phi \\ \varphi(i)=j} }x_\varphi=\sum_{\varphi\in \Phi}x_\varphi =\left\| x \right\|_1. \]
$x$ being nonzero, the above quantity is also nonzero and the $d\times
d$ matrix $(\tilde{x}_{ij}/\left\| x \right\|_1)_{i,j\in \mathcal{I}}$ is 
stochastic and therefore has an invariant measure $a\in
\Delta(\mathcal{I})$:
\[ \sum_{i\in \mathcal{I}}a_i\frac{\tilde{x}_{ij}}{\left\| x \right\|_1}=a_j,\quad j\in \mathcal{I}. \]
Multiplying on both sides by $\left\| x \right\|_1$, we get Equation\eqref{eq:invariant}:
\[ \sum_{i\in \mathcal{I}}a_i\tilde{x}_{ij}=a_j\left\| x
  \right\|_1=a_j \sum_{i\in \mathcal{I}} \sum_{\substack{\varphi\in
      \Phi \\ \varphi(j)=i} }x_\varphi=a_j\sum_{i\in \mathcal{I}}\tilde{x}_{ji},\quad j\in \mathcal{J}.  \]

Let $v\in [-1,1]^d$ and compute $\left< r(a,v) , x \right>$:
\begin{align*}
\left< r(a,v) , x \right>&=\sum_{\varphi \in \Phi} x_{\varphi }\left(  \sum_{i\in \mathcal{I}}a_i(v_{\varphi (i)}-v_i) \right)=\sum_{i,j\in \mathcal{I}}a_i(v_j-v_i) \sum_{\substack{\varphi \in \Phi\\\varphi (i)=j}}x_{\varphi }\\
&=\sum_{i,j\in \mathcal{I}}a_i(v_j-v_i)\tilde{x}_{ij}=\sum_{j\in \mathcal{I}}v_j\sum_{i\in \mathcal{I}}a_i\tilde{x}_{ij}-\sum_{i,j\in \mathcal{I}}a_iv_i\tilde{x}_{ij}\\
&=\sum_{j\in \mathcal{I}}v_ja_j\sum_{i\in \mathcal{I}}\tilde{x}_{ji}-\sum_{i,j\in \mathcal{I}}a_iv_i\tilde{x}_{ij}=0,
\end{align*}
where we used Equation~\eqref{eq:invariant} for the fifth equality.
In particular, $\left< r(a,v) , x \right> \leqslant 0$ and
$\mathbb{R}_-^{\Phi}$ is indeed a B-set for
the game with mixed actions $(\mathcal{I},[-1,1]^d,r)$.
\end{proof}

As for the generator, we choose \(\mathcal{X}=\Delta(\Phi)\) which is a generator of
\((\mathbb{R}_-^{\Phi})^\circ\) thanks to Proposition~\ref{prop:generator-examples}.
Then the support function of \(\Delta(\Phi)\) evaluated at the average
payoff is equal to the average \(\Phi\)-regret:
\begin{align*}
I_{\Delta(\Phi)}^*(\bar{r}_T)&=\frac{1}{T}I_{\Delta(\Phi)}^*\left( \sum_{t=1}^{T}r(i_t,v_t) \right)=\frac{1}{T}\max_{x\in \Delta(\Phi)}\left< \sum_{t=1}^{T}\left( v_{t\varphi (i_t)}-v_{ti_t} \right)_{\varphi \in \Phi}  , x \right>\\ 
&=\frac{1}{T}\max_{\varphi \in \Phi}\sum_{t=1}^{T} \left( v_{t\varphi (i_t)}-v_{ti_t} \right)=\frac{1}{T}\left( \max_{\varphi \in \Phi}\sum_{t=1}^{T} v_{t\varphi (i_t)}-\sum_{t=1}^{T}v_{ti_t}  \right)=\frac{1}{T}\operatorname{Reg}_T^{\Phi}.
\end{align*}
On the simplex \(\Delta(\Phi)\), we choose the entropic regularizer
presented in Section~\ref{sec:regularizers}:
\[ h_{\mathrm{ent}}(x)= \begin{cases}
\displaystyle \sum_{\varphi \in \Phi}x_{\varphi }\log x_{\varphi }&\text{if $x\in
\Delta(\Phi)$}\\
+\infty&\text{otherwise}.
\end{cases}
 \]

Then, the algorithm associated with regularizer \(h_{\mathrm{ent}}\), a
\((\mathcal{I},[-1,1]^d,r,\mathbb{R}_-^{\Phi})\)-oracle \(a\) and
a sequence of positive parameters
\((\eta_t)_{t\geqslant 1}\) is the following. For \(t\geqslant 1\),
\begin{align*}
\text{compute}\quad x_{t\varphi}  &=\frac{\exp \left( \eta_{t-1}\sum_{s=1}^{t-1}r_{s\varphi}
  \right)}{\sum_{\varphi '\in \Phi}\exp \left( \eta_{t-1}\sum_{s=1}^{t-1}r_{s\varphi'} \right)},\quad \varphi \in \Phi\\
\text{choose}\quad a_t&=a(x_t)\\
\text{draw}\quad i_t&\sim a_t\\
\text{observe}\quad r_t&=r(i_t,v_t)=\left( v_{t\varphi (i_t)}-v_{ti_t} \right)_{\varphi \in \Phi}.
\end{align*}
The expression of \(x_t\) is explicit and straightforward and the
computation of mixed action \(a_t=a(x_t)\) via oracle \(a\) consists,
as shown in the proof of Proposition ~\ref{prop:swap-B-set}, in
finding an invariant measure of a \(d\times d\) stochastic matrix,
which can be done efficiently. However, the computation of $x_t$
requires to work with $\left| \Phi \right|$ components, which can be up to
$d^d$. The algorithm from \cite{blum2005external} is much more
efficient computationally as its computational cost is polynomial in $d$.
\begin{theorem}
\label{thm:swap-regret}
Against any sequence $(v_t)_{t\geqslant 1}$ in $[-1,1]^d$ chosen by the Environment, the above algorithm with parameters $\eta_t=\sqrt{\log \left| \Phi \right| /4t}$ (for $t\geqslant 1$) guarantees
\[ \mathbb{E}\left[ \operatorname{Reg}_T^{\Phi}\right] \leqslant 4\sqrt{T\log \left| \Phi \right| }. \]
Let $\delta\in (0,1)$. With probability higher than $1-\delta$, we have
\[ \frac{1}{T}\operatorname{Reg}_T^{\Phi}\leqslant \frac{1}{\sqrt{T}}\left( 4\sqrt{\log \left| \Phi \right| }+2\sqrt{2\log (1/\delta)} \right).  \]
Almost-surely,
\[ \limsup_{T\to +\infty}\frac{1}{T}\operatorname{Reg}_T^{\Phi}\leqslant 0. \]
\end{theorem}
\begin{proof}
For every payoff vector $v\in [-1,1]^d$ and pure action $i\in \mathcal{I}$,
\[ \left\| r(i,v) \right\|_{\infty}=\left\| (v_{\varphi (i)}-v_i)_{\varphi \in \Phi} \right\|_{\infty}\leqslant 2. \]
The result then follows from Theorem~\ref{thm:mirror-descent-approachability-mixed-actions} applied with $M=2$, $K=1$ and the properties of 
regularizer $h_{\mathrm{ent}}$ given by Proposition~\ref{prop:entropy}.
\end{proof}

An important special case is when
\(\Phi\) is the set of all transpositions of \(\mathcal{I}\), in other
words, the set of maps \(\varphi :\mathcal{I}\to \mathcal{I}\) such that
there exists \(i\neq j\) in \(\mathcal{I}\) such that
\[ \varphi (i)=j,\quad \varphi (j)=i,\quad \text{and}\quad \varphi (k)=k \text{ for
    all $k\not \in \left\{ i,j \right\} $}. \]
The
\(\Phi\)-regret is then called the \emph{internal regret} and can be written
\[ \max_{i,j\in \mathcal{I}}\sum_{t=1}^T\mathbbm{1}_{\left\{ i_t=i \right\} }(v_{tj}-v_{ti}). \]
Since \(\left| \Phi \right| =d(d-1)\) in this case, Theorem~\ref{thm:swap-regret} assures that the corresponding algorithm
guarantees a bound on the internal regret of order \(\sqrt{T\log d}\).


%% file: refined-approachability.bbl
\begin{thebibliography}{51}
\providecommand{\natexlab}[1]{#1}
\providecommand{\url}[1]{\texttt{#1}}
\expandafter\ifx\csname urlstyle\endcsname\relax
  \providecommand{\doi}[1]{doi: #1}\else
  \providecommand{\doi}{doi: \begingroup \urlstyle{rm}\Url}\fi

\bibitem[Abernethy et~al.(2011)Abernethy, Bartlett, and
  Hazan]{abernethy2011blackwell}
J.~Abernethy, P.~L. Bartlett, and E.~Hazan.
\newblock Blackwell approachability and low-regret learning are equivalent.
\newblock In \emph{JMLR: Workshop and Conference Proceedings (COLT)},
  volume~19, pages 27--46, 2011.

\bibitem[Azar et~al.(1993)Azar, Kalyanasundaram, Plotkin, Pruhs, and
  Waarts]{azar1993online}
Y.~Azar, B.~Kalyanasundaram, S.~Plotkin, K.~R. Pruhs, and O.~Waarts.
\newblock Online load balancing of temporary tasks.
\newblock In \emph{Workshop on algorithms and data structures}, pages 119--130.
  Springer, 1993.

\bibitem[Azar et~al.(2014)Azar, Felge, Feldman, and
  Tennenholtz]{azar2014sequential}
Y.~Azar, U.~Felge, M.~Feldman, and M.~Tennenholtz.
\newblock Sequential decision making with vector outcomes.
\newblock In \emph{Proceedings of the 5th conference on Innovations in
  theoretical computer science}, pages 195--206, 2014.

\bibitem[Bernstein and Shimkin(2015)]{bernstein2015response}
A.~Bernstein and N.~Shimkin.
\newblock Response-based approachability with applications to generalized
  no-regret problems.
\newblock \emph{The Journal of Machine Learning Research}, 16\penalty0
  (1):\penalty0 747--773, 2015.

\bibitem[Blackwell(1954)]{blackwell1954controlled}
D.~Blackwell.
\newblock Controlled random walks.
\newblock In \emph{Proceedings of the International Congress of
  Mathematicians}, volume~3, pages 336--338, 1954.

\bibitem[Blackwell(1956)]{blackwell1956analog}
D.~Blackwell.
\newblock An analog of the minimax theorem for vector payoffs.
\newblock \emph{Pacific Journal of Mathematics}, 6\penalty0 (1):\penalty0 1--8,
  1956.

\bibitem[Blum and Mansour(2005)]{blum2005external}
A.~Blum and Y.~Mansour.
\newblock From external to internal regret.
\newblock In \emph{Learning Theory}, pages 621--636. Springer, 2005.

\bibitem[Borodin and El-Yaniv(1998)]{borodin1998online}
A.~Borodin and R.~El-Yaniv.
\newblock \emph{Online computation and competitive analysis}.
\newblock Cambridge University Press, 1998.

\bibitem[Borwein and Lewis(2010)]{borwein2010convex}
J.~M. Borwein and A.~S. Lewis.
\newblock \emph{Convex Analysis and Nonlinear Optimization: Theory and
  Examples}.
\newblock Springer, 2010.

\bibitem[Boyd and Vandenberghe(2004)]{boyd2004convex}
S.~Boyd and L.~Vandenberghe.
\newblock \emph{Convex optimization}.
\newblock Cambridge University Press, 2004.

\bibitem[Bubeck(2011)]{bubeck2011introduction}
S.~Bubeck.
\newblock \emph{Introduction to Online Optimization: Lecture Notes}.
\newblock Princeton University, 2011.

\bibitem[Cesa-Bianchi(1997)]{cesa1997analysis}
N.~Cesa-Bianchi.
\newblock Analysis of two gradient-based algorithms for on-line regression.
\newblock In \emph{Proceedings of the Tenth Annual Conference on Computational
  Learning Theory (COLT)}, pages 163--170. ACM, 1997.

\bibitem[Cesa-Bianchi and Lugosi(2006)]{cesa2006prediction}
N.~Cesa-Bianchi and G.~Lugosi.
\newblock \emph{Prediction, Learning, and Games}.
\newblock Cambridge University Press, 2006.

\bibitem[Dawid(1982)]{dawid1982well}
A.~P. Dawid.
\newblock The well-calibrated bayesian.
\newblock \emph{Journal of the American Statistical Association}, 77\penalty0
  (379):\penalty0 605--610, 1982.

\bibitem[Even-Dar et~al.(2009)Even-Dar, Kleinberg, Mannor, and
  Mansour]{even2009online}
E.~Even-Dar, R.~Kleinberg, S.~Mannor, and Y.~Mansour.
\newblock Online learning for global cost functions.
\newblock In \emph{COLT}, 2009.

\bibitem[Farina et~al.(2021)Farina, Kroer, and Sandholm]{farina2020faster}
G.~Farina, C.~Kroer, and T.~Sandholm.
\newblock Faster game solving via predictive {B}lackwell approachability:
  Connecting regret matching and mirror descent.
\newblock \emph{Proceedings of the AAAI Conference on Artificial Intelligence},
  35\penalty0 (6):\penalty0 5363--5371, May 2021.

\bibitem[Foster and Vohra(1997)]{foster1997calibrated}
D.~P. Foster and R.~V. Vohra.
\newblock Calibrated learning and correlated equilibrium.
\newblock \emph{Games and Economic Behavior}, 21\penalty0 (1):\penalty0 40--55,
  1997.

\bibitem[Foster and Vohra(1998)]{foster1998asymptotic}
D.~P. Foster and R.~V. Vohra.
\newblock Asymptotic calibration.
\newblock \emph{Biometrika}, 85\penalty0 (2):\penalty0 379--390, 1998.

\bibitem[Fudenberg and Levine(1995)]{fudenberg1995consistency}
D.~Fudenberg and D.~K. Levine.
\newblock Consistency and cautious fictitious play.
\newblock \emph{Journal of Economic Dynamics and Control}, 19\penalty0
  (5):\penalty0 1065--1089, 1995.

\bibitem[Fudenberg and Levine(1999)]{fudenberg1999conditional}
D.~Fudenberg and D.~K. Levine.
\newblock Conditional universal consistency.
\newblock \emph{Games and Economic Behavior}, 29\penalty0 (1):\penalty0
  104--130, 1999.

\bibitem[Gentile and Warmuth(1998)]{gentile1998linear}
C.~Gentile and M.~K. Warmuth.
\newblock Linear hinge loss and average margin.
\newblock In \emph{Advances in Neural Information Processing Systems (NIPS)},
  volume~11, pages 225--231, 1998.

\bibitem[Gordon(2007)]{gordon2007no}
G.~J. Gordon.
\newblock No-regret algorithms for online convex programs.
\newblock In \emph{Advances in Neural Information Processing Systems}, pages
  489--496, 2007.

\bibitem[Grove et~al.(2001)Grove, Littlestone, and
  Schuurmans]{grove2001general}
A.~J. Grove, N.~Littlestone, and D.~Schuurmans.
\newblock General convergence results for linear discriminant updates.
\newblock \emph{Machine Learning}, 43\penalty0 (3):\penalty0 173--210, 2001.

\bibitem[Hannan(1957)]{hannan1957approximation}
J.~Hannan.
\newblock Approximation to {B}ayes risk in repeated play.
\newblock \emph{Contributions to the Theory of Games}, 3\penalty0
  (97--139):\penalty0 2, 1957.

\bibitem[Hart and Mas-Colell(2000)]{hart2000simple}
S.~Hart and A.~Mas-Colell.
\newblock A simple adaptive procedure leading to correlated equilibrium.
\newblock \emph{Econometrica}, 68:\penalty0 1127--1150, 2000.

\bibitem[Hart and Mas-Colell(2001)]{hart2001general}
S.~Hart and A.~Mas-Colell.
\newblock A general class of adaptive strategies.
\newblock \emph{Journal of Economic Theory}, 98\penalty0 (1):\penalty0 26--54,
  2001.

\bibitem[Hazan et~al.(2010)Hazan, Kale, and Warmuth]{hazan2010learning}
E.~Hazan, S.~Kale, and M.~K. Warmuth.
\newblock Learning rotations with little regret.
\newblock In \emph{Proceedings of the 23rd Conference on Learning Theory
  (COLT)}, pages 144--154, 2010.

\bibitem[Helmbold and Warmuth(2009)]{helmbold2009learning}
D.~P. Helmbold and M.~K. Warmuth.
\newblock Learning permutations with exponential weights.
\newblock \emph{The Journal of Machine Learning Research}, 10:\penalty0
  1705--1736, 2009.

\bibitem[Kakade et~al.(2012)Kakade, Shalev-Shwartz, and
  Tewari]{kakade2012regularization}
S.~M. Kakade, S.~Shalev-Shwartz, and A.~Tewari.
\newblock Regularization techniques for learning with matrices.
\newblock \emph{The Journal of Machine Learning Research}, 13\penalty0
  (1):\penalty0 1865--1890, 2012.

\bibitem[Kalai and Vempala(2005)]{kalai2005efficient}
A.~Kalai and S.~Vempala.
\newblock Efficient algorithms for online decision problems.
\newblock \emph{Journal of Computer and System Sciences}, 71\penalty0
  (3):\penalty0 291--307, 2005.

\bibitem[Kivinen and Warmuth(2001)]{kivinen2001relative}
J.~Kivinen and M.~K. Warmuth.
\newblock Relative loss bounds for multidimensional regression problems.
\newblock \emph{Machine Learning}, 45\penalty0 (3):\penalty0 301--329, 2001.

\bibitem[Kohlberg(1975)]{kohlberg1975optimal}
E.~Kohlberg.
\newblock Optimal strategies in repeated games with incomplete information.
\newblock \emph{International Journal of Game Theory}, 4\penalty0 (1):\penalty0
  7--24, 1975.

\bibitem[Koolen et~al.(2010)Koolen, Warmuth, and Kivinen]{koolen2010hedging}
W.~M. Koolen, M.~K. Warmuth, and J.~Kivinen.
\newblock Hedging structured concepts.
\newblock In \emph{Proceedings of the 23rd Conference on Learning Theory
  (COLT)}, pages 93--105, 2010.

\bibitem[Liu et~al.(2021)Liu, Hatano, and Takimoto]{liu2020improved}
Y.~Liu, K.~Hatano, and E.~Takimoto.
\newblock Improved algorithms for online load balancing.
\newblock In \emph{SOFSEM 2021: Theory and Practice of Computer Science}, pages
  203--217. Springer International Publishing, 2021.

\bibitem[Mannor and Shimkin(2008)]{mannor2008regret}
S.~Mannor and N.~Shimkin.
\newblock Regret minimization in repeated matrix games with variable stage
  duration.
\newblock \emph{Games and Economic Behavior}, 63\penalty0 (1):\penalty0
  227--258, 2008.

\bibitem[Mannor et~al.(2014)Mannor, Perchet, and
  Stoltz]{mannor2014approachability}
S.~Mannor, V.~Perchet, and G.~Stoltz.
\newblock Approachability in unknown games: Online learning meets
  multi-objective optimization.
\newblock In \emph{Conference on Learning Theory}, pages 339--355, 2014.

\bibitem[Molinaro(2017)]{molinaro2017online}
M.~Molinaro.
\newblock Online and random-order load balancing simultaneously.
\newblock In \emph{Proceedings of the Twenty-Eighth Annual ACM-SIAM Symposium
  on Discrete Algorithms}, pages 1638--1650. SIAM, 2017.

\bibitem[Moreau(1962)]{moreau1962decomposition}
J.-J. Moreau.
\newblock D{\'e}composition orthogonale d'un espace hilbertien selon deux
  c{\^o}nes mutuellement polaires.
\newblock \emph{Comptes rendus de l'Acad{\'e}mie des Sciences}, 255:\penalty0
  238--240, 1962.

\bibitem[Nesterov(2009)]{nesterov2009primal}
Y.~Nesterov.
\newblock Primal-dual subgradient methods for convex problems.
\newblock \emph{Mathematical programming}, 120\penalty0 (1):\penalty0 221--259,
  2009.

\bibitem[Perchet(2014)]{perchet2014approachability}
V.~Perchet.
\newblock Approachability, regret and calibration: Implications and
  equivalences.
\newblock \emph{Journal of Dynamics and Games}, 1\penalty0 (2):\penalty0
  181--254, 2014.

\bibitem[Perchet(2015)]{perchet2015exponential}
V.~Perchet.
\newblock Exponential weight approachability, applications to calibration and
  regret minimization.
\newblock \emph{Dynamic Games and Applications}, 5\penalty0 (1):\penalty0
  136--153, 2015.

\bibitem[Rakhlin et~al.(2011)Rakhlin, Sridharan, and Tewari]{rakhlin2011online}
A.~Rakhlin, K.~Sridharan, and A.~Tewari.
\newblock Online learning: {B}eyond regret.
\newblock In \emph{Proceedings of the 24th Annual Conference on Learning
  Theory}, pages 559--594, 2011.

\bibitem[Rockafellar(1970)]{rockafellar1970convex}
R.~T. Rockafellar.
\newblock \emph{Convex Analysis}.
\newblock Princeton University Press, 1970.

\bibitem[Shalev-Shwartz(2007)]{shalev2007online}
S.~Shalev-Shwartz.
\newblock \emph{Online learning: Theory, Algorithms, and Applications}.
\newblock PhD thesis, The Hebrew University of Jerusalem, 2007.

\bibitem[Shalev-Shwartz(2011)]{shalev2011online}
S.~Shalev-Shwartz.
\newblock Online learning and online convex optimization.
\newblock \emph{Foundations and Trends in Machine Learning}, 4\penalty0
  (2):\penalty0 107--194, 2011.

\bibitem[Shimkin(2016)]{shimkin2016online}
N.~Shimkin.
\newblock An online convex optimization approach to {B}lackwell's
  approachability.
\newblock \emph{The Journal of Machine Learning Research}, 17\penalty0
  (1):\penalty0 4434--4456, 2016.

\bibitem[Stoltz and Lugosi(2005)]{stoltz2005internal}
G.~Stoltz and G.~Lugosi.
\newblock Internal regret in on-line portfolio selection.
\newblock \emph{Machine Learning}, 59\penalty0 (1-2):\penalty0 125--159, 2005.

\bibitem[Takimoto and Warmuth(2003)]{takimoto2003path}
E.~Takimoto and M.~K. Warmuth.
\newblock Path kernels and multiplicative updates.
\newblock \emph{The Journal of Machine Learning Research}, 4:\penalty0
  773--818, 2003.

\bibitem[Tammelin et~al.(2015)Tammelin, Burch, Johanson, and
  Bowling]{tammelin2015solving}
O.~Tammelin, N.~Burch, M.~Johanson, and M.~Bowling.
\newblock Solving heads-up limit texas hold'em.
\newblock In \emph{Twenty-fourth international joint conference on artificial
  intelligence}, 2015.

\bibitem[Warmuth and Kuzmin(2008)]{warmuth2008randomized}
M.~K. Warmuth and D.~Kuzmin.
\newblock Randomized online {PCA} algorithms with regret bounds that are
  logarithmic in the dimension.
\newblock \emph{Journal of Machine Learning Research}, 9\penalty0
  (10):\penalty0 2287--2320, 2008.

\bibitem[Zinkevich et~al.(2007)Zinkevich, Johanson, Bowling, and
  Piccione]{zinkevich2007regret}
M.~Zinkevich, M.~Johanson, M.~Bowling, and C.~Piccione.
\newblock Regret minimization in games with incomplete information.
\newblock \emph{Advances in neural information processing systems},
  20:\penalty0 1729--1736, 2007.

\end{thebibliography}
